\documentclass[sigconf]{acmart}
\usepackage{xspace,balance,tabularx,multirow}
\usepackage{flushend}
\usepackage{tikz}
\usepackage{pgfplots}
\pgfplotsset{compat=1.16}
\usetikzlibrary{patterns}
\usepackage{subfig}
\usepackage[ruled, vlined, linesnumbered]{algorithm2e}
\usepackage{xcolor}
\usepackage{colortbl}
\usepackage{bbold}
\SetKwComment{Comment}{$\triangleright$\ }{}
\usepackage{enumitem}
\usepackage{tablefootnote}
\usepackage{upgreek,textgreek}
\usepackage{pifont}
\usepackage[noabbrev]{cleveref}
\usepackage{titlecaps}
\usepackage{lipsum}

\pgfplotsset{every tick label/.append style={font=\tiny}}

\newlength{\starsize}
\newlength{\starspread}
\tikzset{starsize/.code={\setlength{\starsize}{#1}},
         starspread/.code={\setlength{\starspread}{#1}}}
\tikzset{starsize=1mm,
         starspread=3mm}
\pgfdeclarepatternformonly[\starspread,\starsize]
  {my fivepointed stars}
  {\pgfpointorigin}
  {\pgfqpoint{\starspread}{\starspread}}
  {\pgfqpoint{\starspread}{\starspread}}
  {
   \pgftransformshift{\pgfqpoint{\starsize}{\starsize}}
   \pgfpathmoveto{\pgfqpointpolar{18}{\starsize}}
   \pgfpathlineto{\pgfqpointpolar{162}{\starsize}}
   \pgfpathlineto{\pgfqpointpolar{306}{\starsize}}
   \pgfpathlineto{\pgfqpointpolar{90}{\starsize}}
   \pgfpathlineto{\pgfqpointpolar{234}{\starsize}}
   \pgfpathclose%
   \pgfusepath{fill}
  }

\makeatletter
\newcommand*\bigcdot{\mathpalette\bigcdot@{.5}}
\newcommand*\bigcdot@[2]{\mathbin{\vcenter{\hbox{\scalebox{#2}{$\m@th#1\bullet$}}}}}
\makeatother

\newcommand{\bigzero}{\mbox{\normalfont\large\bfseries 0}}
\newcommand{\rvline}{\hspace*{-\arraycolsep}\vline\hspace*{-\arraycolsep}}

\def\header{\vspace{1mm} \noindent}

\newcommand{\ie}{{\it i.e.},\xspace}

\newcommand{\U}{\mathcal{U}\xspace}
\newcommand{\V}{\mathcal{V}\xspace}

\newcommand{\G}{\mathcal{G}\xspace}

\newcommand{\EDG}{\mathcal{E}\xspace}
\newcommand{\OO}{\mathcal{O}\xspace}

\newcommand{\DM}{\mathbf{D}\xspace}
\newcommand{\IM}{\mathbf{I}\xspace}

\newcommand{\MM}{\mathbf{M}\xspace}
\newcommand{\PM}{\mathbf{P}\xspace}

\newcommand{\YM}{\mathbf{Y}\xspace}
\newcommand{\XM}{\mathbf{X}\xspace}

\newcommand{\EM}{\mathbf{E}\xspace}
\newcommand{\ZM}{\mathbf{Z}\xspace}

\newcommand{\UM}{\mathbf{U}\xspace}
\newcommand{\VM}{\mathbf{V}\xspace}
\newcommand{\SVM}{\boldsymbol{\Sigma}\xspace}

\newcommand{\alg}{EAGLE}
\newcommand{\algo}{\textsc{\alg}\xspace}

\newcommand{\ffp}{\textsc{FFP}\xspace}
\newcommand{\dvffp}{\textsc{DV-FFP}\xspace}

\newcommand{\QM}{\mathbf{Q}\xspace}

\newenvironment{customlegend}[1][]{%
    \begingroup
    \csname pgfplots@init@cleared@structures\endcsname
    \pgfplotsset{#1}%
}{%
    \csname pgfplots@createlegend\endcsname
    \endgroup
}%

\def\addlegendimage{\csname pgfplots@addlegendimage\endcsname}

\makeatletter
\newcommand\footnoteref[1]{\protected@xdef\@thefnmark{\ref{#1}}\@footnotemark}
\makeatother

\let\oldnl\nl
\newcommand{\nonl}{\renewcommand{\nl}{\let\nl\oldnl}}

\SetKwComment{Comment}{/* }{ */}

\makeatletter 
\g@addto@macro{\@algocf@init}{\SetKwInOut{Parameter}{Parameters}} 
\makeatother

\definecolor{myred}{HTML}{fd7f6f}
\definecolor{myred_new}{HTML}{D8D8D8}
\definecolor{myred_new2}{HTML}{D7191C}
\definecolor{myblue}{HTML}{7eb0d5}
\definecolor{mygreen}{HTML}{b2e061}
\definecolor{mypurple}{HTML}{bd7ebe}
\definecolor{myorange}{HTML}{ffb55a}
\definecolor{myyellow}{HTML}{ffee65}
\definecolor{mypurple2}{HTML}{beb9db}
\definecolor{mypink}{HTML}{fdcce5}
\definecolor{mycyan}{HTML}{8bd3c7}

\definecolor{myblue2}{HTML}{115f9a}
\definecolor{myred2}{HTML}{c23728}

\definecolor{my_purple}{HTML}{4dbeee}
\definecolor{my_teal}{HTML}{77ac30}

\definecolor{RYB1}{RGB}{141, 211, 199}
\definecolor{RYB2}{RGB}{255, 255, 179}
\definecolor{RYB3}{RGB}{190, 186, 218}
\definecolor{RYB4}{RGB}{251, 128, 114}
\definecolor{RYB5}{RGB}{128, 177, 211}
\definecolor{RYB6}{RGB}{253, 180, 98}
\definecolor{RYB7}{RGB}{179, 222, 105}

\definecolor{B0}{HTML}{3C2F80}
\definecolor{B1}{HTML}{012030}
\definecolor{B2}{HTML}{0162A7}
\definecolor{B3}{HTML}{36A7CF}
\definecolor{B4}{HTML}{9AEBA3}
\definecolor{B5}{HTML}{DAFDBA}
\definecolor{B6}{HTML}{45C4B0}

\definecolor{O1}{HTML}{F29E38}
\definecolor{O2}{HTML}{F28444}
\definecolor{O3}{HTML}{D53E0F}

\definecolor{R1}{HTML}{F2889B}

\AtBeginDocument{%
  \providecommand\BibTeX{{%
    \normalfont B\kern-0.5em{\scshape i\kern-0.25em b}\kern-0.8em\TeX}}}

\setcopyright{acmcopyright}
\copyrightyear{2024} 
\acmYear{2024} 
\setcopyright{rightsretained} 
\acmConference[KDD '24]{Proceedings of the 30th ACM SIGKDD Conference on Knowledge Discovery and Data Mining}{August 25--29, 2024}{Barcelona, Spain}
\acmBooktitle{Proceedings of the 30th ACM SIGKDD Conference on Knowledge Discovery and Data Mining (KDD '24), August 25--29, 2024, Barcelona, Spain}\acmDOI{10.1145/3637528.3671805}
\acmISBN{979-8-4007-0490-1/24/08}

\begin{document}
\title{Effective Edge-wise Representation Learning in Edge-Attributed Bipartite Graphs}

\author{Hewen Wang}
\affiliation{%
  \institution{National University of Singapore}
  \city{Singapore}
  \country{Singapore}
}
\email{wanghewen@u.nus.edu}
\orcid{0000-0002-9757-4347}

\author{Renchi Yang}
\affiliation{%
  \institution{Hong Kong Baptist University}
  \city{Hong Kong SAR}
  \country{China}
}
\email{renchi@hkbu.edu.hk}
\orcid{0000-0002-7284-3096}

\author{Xiaokui Xiao}
\affiliation{%
  \institution{National University of Singapore}
  \city{Singapore}
  \country{Singapore}
}
\email{xkxiao@nus.edu.sg}
\orcid{0000-0003-0914-4580}

\renewcommand{\shortauthors}{Trovato and Tobin, et al.}

\begin{abstract}
Graph representation learning (GRL) is to encode graph elements into informative vector representations, which can be used in downstream tasks for analyzing graph-structured data and has seen extensive applications in various domains. 
However, the majority of extant studies on GRL are geared towards generating node representations, which cannot be readily employed to perform edge-based analytics tasks in {\em edge-attributed bipartite graphs} (EABGs) that pervade the real world, e.g., spam review detection in customer-product reviews and identifying fraudulent transactions in user-merchant networks. Compared to node-wise GRL, learning edge representations (ERL) on such graphs is challenging due to the need to incorporate the structure and attribute semantics from the perspective of edges while considering the separate influence of two heterogeneous node sets $\U$ and $\V$ in bipartite graphs. To our knowledge, despite its importance, limited research has been devoted to this frontier, and existing workarounds all suffer from sub-par results.


Motivated by this, this paper designs \algo, an effective ERL method for EABGs. Building on an in-depth and rigorous theoretical analysis, we propose the {\em factorized feature propagation} (FFP) scheme for edge representations with adequate incorporation of long-range dependencies of edges/features without incurring tremendous computation overheads.
We further ameliorate FFP as a dual-view FFP by taking into account the influences from nodes in $\U$ and $\V$ severally in ERL. Extensive experiments on 5 real datasets showcase the effectiveness of the proposed \algo models in semi-supervised edge classification tasks. In particular, \algo can attain a considerable gain of at most $38.11\%$ in AP and $1.86\%$ in AUC when compared to the best baselines. 
\end{abstract}





\maketitle

\section{Introduction}
{\em Edge-attributed bipartite graphs} (EABGs) (a.k.a. attributed interaction graphs \cite{zhang2017learning}) are an expressive data structure used to model the interactive behaviors between two sets of objects $\U$ and $\V$ where the behaviors are characterized by rich attributes. Practical examples of EABGs include reviews from users/reviewers on movies, businesses, products, and papers; transactions between users and merchants; and disease-protein associations.

In real life, EABGs have seen widespread use in detecting spam reviews in e-commerce \cite{yu2023mrfs}, malicious incidents in telecommunication networks \cite{yan2018telecomm}, fraudulent transactions/accounts in finance \cite{wang2022review} or E-payment systems \cite{li2021temporal}, abusive behaviors in online retail websites  \cite{wang2021bipartite}, insider threats from audit events \cite{garchery2020adsage}, and others \cite{fathony2023interaction,wu2017efficient}. The majority of such applications can be framed as edge-based prediction or classification tasks in EABGs.



In recent years, graph representation learning (e.g., graph neural networks and network embedding) has emerged as a popular and powerful technique for graph analytics and has seen fruitful success in various domains \cite{hamilton2017representation}. In a nutshell, GRL seeks to map graph elements in the input graph into feature representations (a.k.a. embeddings), based on which we can perform downstream prediction or classification tasks. However, to our knowledge, most of the existing GRL models, e.g., GCN \cite{Kipf2016}, GraphSAGE \cite{Hamilton2017} and GAT \cite{Velickovivelickovic}, are devised for learning node-wise representations in node-attributed graphs, and {\em edge-wise representation learning} (ERL), especially on EABGs, is as of yet under-explored.
A common treatment for obtaining edge representations is to directly apply the canonical node-wise GRL models \cite{Kipf2016,Hamilton2017,Velickovivelickovic,Wu2019,velivckovic2018deep,brody2022how,huang2023node} to generate node embeddings, followed by concatenating them as the embeddings of the corresponding edges. Despite its simplicity, this methodology falls short of not only the accurate preservation of the graph topology from the perspective of edges (demanding an effective combination of node embeddings) but also the incorporation of the edge attributes in EABGs, thereby resulting in compromised embedding quality.
Another category of workarounds is to simply transform the original EABGs into node-attributed unipartite graphs by converting the edges into nodes and connecting them if they share common endpoints in the input EABGs. In doing so, the node-wise GRL techniques can be naturally adopted on such projected graphs for deriving edge representations.
Unfortunately, aside from information loss of the bipartite structure in the input EABG $\G$ by the simple transformation \cite{zhou2007bipartite,yang2024efficient}, such projection-based approaches rely on constructing an edge-to-edge graph $\G^{\prime}$, which often entail immense space consumption (up to $O(m^2)$ in the worst case) due to the scale-free property of real-world graph $\G$, i.e., a few nodes connecting to a significant amount of nodes in $\G$ and creating a multitude of edge-to-edge associations in $\G^{\prime}$~\cite{yang2022efficient}.
Recently, several efforts \cite{wang2023efficient,bielak2022attre2vec,jo2021edge,gao2019edge2vec} have been specifically invested towards learning edge-wise feature representations. However, these ERL models are either designed for unipartite graphs or hypergraphs and hence, cannot readily be applied to EABGs for high-quality representations, as they are unable to capture the unique characteristics of bipartite graphs, particularly the underlying semantics of connections to nodes in $\U$ and $\V$ from two heterogeneous sources.



To remedy the deficiencies of existing works, this paper presents \algo (\underline{E}dge-wise Bip\underline{A}rtite \underline{G}raph Representation \underline{LE}arning) for effective ERL in EABGs. By taking inspiration from the numeric analysis \cite{ma2021unified,yang2021attributes,zhu2021interpreting} of the most popular GRL solutions, i.e., classic message-passing (a.k.a. feature-propagation) GNNs, we begin by formalizing the edge-wise representation learning objective as an optimization problem, while considering the respective influence of edge associations with two sets of heterogeneous nodes $\U$ and $\V$. Through our theoretical insights into the optimal solution to the optimization problem, the derived feature propagation rules, and their connections to the well-established Markov chain theory, we unveil the necessity of preserving long-range dependencies \cite{wu2021representing} of edges in edge representations on EABGs. Based thereon, we propose a {\em factorized feature propagation} (FFP) scheme to enable efficient and effective long-range feature propagation for generating edge representations in \algo. Furthermore, we upgrade \algo with the {\em dual-view factorized feature propagation} (DV-FFP) for flexible and full exploitation of semantics from two sets of nodes $\U$ and $\V$ in EABGs. More precisely, instead of combining edge associations to $\U$ and $\V$ via a given hyperparameter for subsequent ERL, DV-FFP learns two sets of edge embeddings using the connections to $\U$ and $\V$, respectively, followed by an aggregator function that combines them as the final representations. Following previous work, our \algo models are trained by feeding the final edge representations into the loss function for the semi-supervised edge classification.

We evaluate the proposed \algo models against 9 baselines on 5 real EABGs in terms of semi-supervised edge classification tasks. The experimental results exhibit that our \algo models consistently achieve the best empirical performance over 5 datasets with remarkable gains, further validating the effectiveness of FFP and DV-FFP schemes in \algo. 
Notably, on the academic graphs AMiner and OAG dataset, \algo can obtain $38.11\%$ and $11.97\%$ performance gains in terms of average precision (AP) over the best competitor, indicating the superiority of \algo in learning predictive edge representations on EABGs.

The remainder of this paper is structured as follows.
After presenting the preliminaries and formal problem definition in Section \ref{sec:preliminary}, we design the basic \algo model with \ffp in Section \ref{sec:solution}. 
We further introduce an enhanced \algo model with dual-view \ffp in Sections \ref{sec:solution-opt}.
Experiments are conducted in Section \ref{sec:experiments}. Section \ref{sec:relatedwork} reviews related studies, and Section \ref{sec:conclude} concludes the paper.
\section{Preliminaries}\label{sec:preliminary}
Throughout this paper, sets are denoted by calligraphic letters, e.g., $\V$, and $|\V|$ is used to denote the cardinality of the set $\V$. Matrices (resp. vectors) are written in bold uppercase (resp. lowercase) letters, e.g., $\MM$ (resp. $\mathbf{x}$). The superscript $\MM^{\top}$ is used to symbolize the transpose of matrix $\MM$. $\MM[i]$ ($\MM[:,i]$) is used to represent the $i$-th row (resp. column) of matrix $\MM$. Accordingly, $\MM[i,j]$ denotes the $(i,j)$-th entry in matrix $\MM$. For each vector $\MM[i]$, we use $\|\MM[i]\|$ to represent its $L_2$ norm, i.e., $\|\MM[i]\|=\sqrt{\sum_{j=1}^{d}{\MM[i,j]^2}}$ and $\|\MM\|_F$ to represent the Frobenius norm of $\MM$, i.e., $\|\MM\|_F=\sqrt{\sum_{i=1}^{n}\sum_{j=1}^{d}{\MM[i,j]^2}}$.

\subsection{Edge-Attributed Bipartite Graphs}
\begin{definition}[\bf Edge-Attributed Bipartite Graphs (EABG)]
An EABG is defined as $\G=(\U\cup \V,\EDG, \XM)$, where $\U$ and $\V$ represent two disjoint node sets, $\EDG$ consists of the inter-set edges connecting nodes in $\U$ and $\V$, and each edge $e_{i}$ is associated with a length-$d$ attribute vector $\XM[i]$.
\end{definition}
\begin{figure}[H]
\vspace{-3ex}
\centering
\includegraphics[width=0.8\columnwidth]{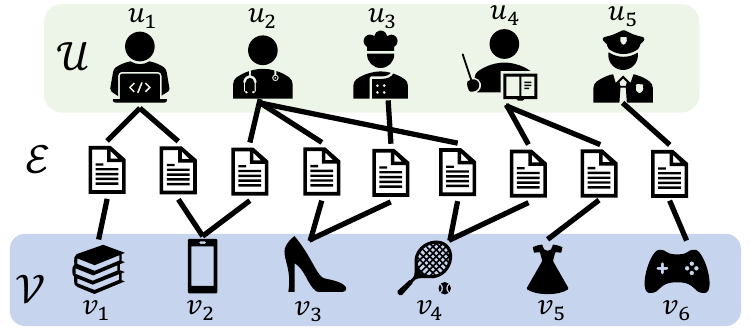}
\vspace{-3mm}
\caption{An Example EABG}\label{fig:toy}
\vspace{-3mm}
\end{figure}
Figure \ref{fig:toy} exemplifies an EABG $\G$ in online retail platforms (e.g., Amazon and eBay) with 5 users $u_1$-$u_5$ in $\U$, 6 products $v_1$-$v_6$ in $\V$, and the user-product interactions in $\EDG$. Each interaction (i.e., edge) is associated with a review (i.e., edge attributes) from the user on the product.

For each node $u_i\in \U$ (resp. $v_j\in \V$), $\EDG_{u_i}$ (resp. $\EDG_{v_j}$) symbolizes the set of edges incident to $u_i$ (resp. $v_j$). We use $\DM_{\U}\in \mathbb{R}^{|\U| \times |\U|}$ (resp. $\DM_{\V}\in \mathbb{R}^{|\V| \times |\V|}$) to represent the diagonal matrix whose diagonal entries correspond to the degrees of nodes in $\U$ (resp. $\V$), e.g., $\DM_{\U}[u_i,u_i]=|\EDG_{u_i}|$ and $\DM_{\V}[v_i,v_i]=|\EDG_{v_i}|$. 
Then, $\DM=\textstyle \begin{pmatrix}
\DM_\U
  & \rvline & 0 \\
\hline
  0 & \rvline &
\DM_\V
\end{pmatrix}\in \mathbb{R}^{(|\U|+|\V|) \times (|\U|+|\V|)}$ is the diagonal node degree matrix of $\G$. 
Further, we denote by $\EM_{\U}\in \mathbb{R}^{|\EDG| \times |\U|}$ and $\EM_{\V}\in \mathbb{R}^{|\EDG| \times |\V|}$ the edge-node indicator matrices for node sets $\U$ and $\V$, respectively. More precisely, for each edge $e_i\in \EDG$ and its two end points $u^{(i)},v^{(i)}$, we have $\EM_{\U}[e_i,u^{(i)}]=\EM_{\V}[e_i,v^{(i)}]=1$. For other nodes $u\in \U \setminus u^{(i)}$ and $v\in \V \setminus v^{(i)}$, $\EM_{\U}[e_i,u]=\EM_{\V}[e_i,v]=0$.

On the basis of $\DM_\U$, $\EM_\U$ and $\DM_\V,\EM_\V$, we define edge-wise transition matrix $\PM_\U$ and $\PM_\V$ as follows:
\begin{equation}\label{eq:PU-PV}
\PM_\U = \EM_{\U} {\DM_{\U}^{-1}} \EM_{\U}^{\top}\ \text{and}\ \PM_\V = \EM_{\V}{\DM_{\V}^{-1}} \EM_{\V}^{\top}.
\end{equation}
Lemma \ref{lem:PUPV} unveils a unique property of $\PM_\U$ and $\PM_\V$, which is crucial to the design of our \algo model.
\begin{lemma}\label{lem:PUPV}
$\PM_\U$ and $\PM_\V$ are doubly stochastic matrices.
\end{lemma}





\subsection{Problem Formulation}
We formalize the {\em edge representation learning} (ERL) in EABGs as follows. Given an EABG $\G=(\U\cup \V, \EDG, \XM)$, the task of ERL aims to build a model $f: \EDG \rightarrow \ZM \in \mathbb{R}^{|\EDG|\times z}$ ($z\ll |\EDG|$), which transforms each edge $e_i\in \EDG$ into a length-$z$ vector $\ZM[e_i]$ as its feature representation. Such a feature representation $\ZM[e_i]$ should capture the rich semantics underlying both the bipartite graph structures and edge attributes. In this paper, we focus on the edge classification task, and thus, the edge representations are learned in a semi-supervised fashion by plugging the loss function for classifying edges into the model $f$.
\begin{figure*}[!t]
    \centering
    \includegraphics[width=0.9\textwidth]{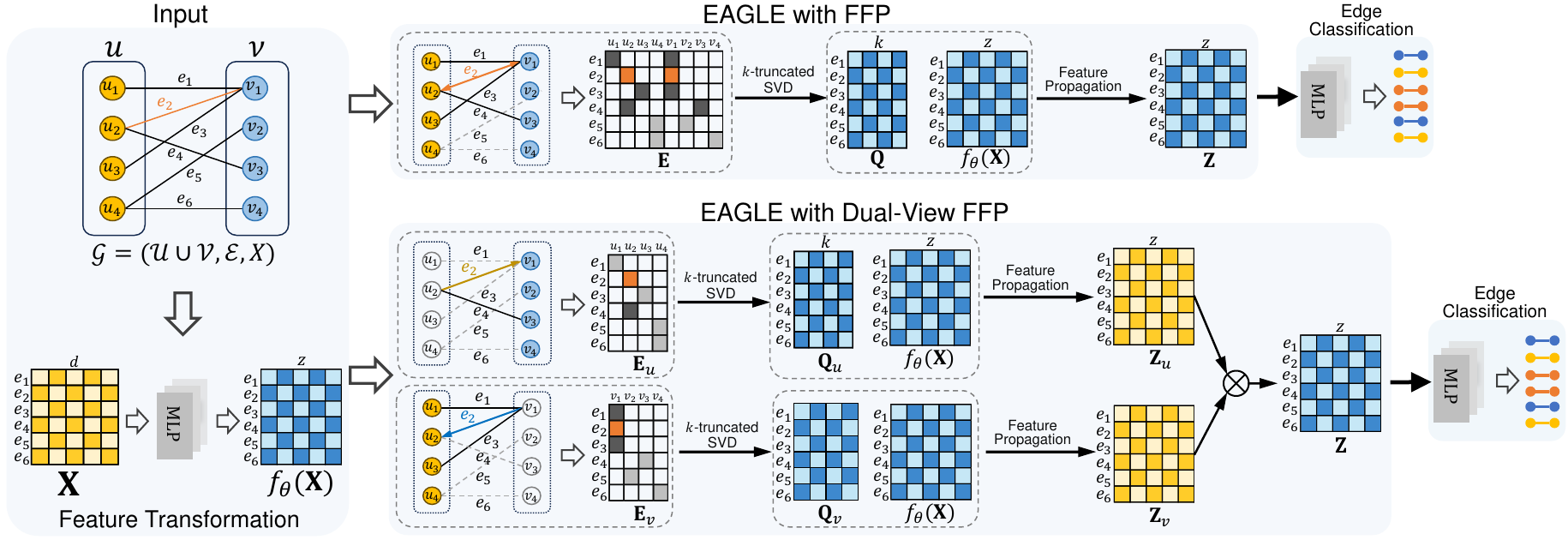}
    \vspace{-3mm}
    \caption{The overall framework of \algo}\label{fig:flow}
    \vspace{-2mm}
\end{figure*}

\section{The \algo Model}\label{sec:solution}
As illustrated in Figure \ref{fig:flow}, we have developed two ERL models for EABG, i.e., \algo with \ffp and dual-view \ffp, both of which involve two key steps: $k$-truncated singular value decomposition (SVD) and feature propagation.

In this section, we focus on introducing our base \algo model, i.e., \algo with \ffp.
Section~\ref{sec:obj} first presents the objective of learning the edge-wise representations, while Section~\ref{sec:sol} then offers an in-depth analysis of the solution to the optimization objective.
In Section~\ref{sec:ffp}, we elaborate on the feature propagation mechanism for computing the edge representations, followed by the loss function for the model training in Section \ref{sec:loss}.

\subsection{Representation Learning Objective}\label{sec:obj}
Inspired by the numeric optimization analysis of generalized graph neural network models in recent studies \cite{ma2021unified,yang2021attributes,zhu2021interpreting}, we formulate the ERL in EABGs as an optimization problem with consideration of the lopsided nature of bipartite graphs.

More concretely, \algo aims at achieving two goals: (i) the edge representations $\ZM$ close to the input edge feature matrix; and (ii) representations of edges that are incident to the same nodes should be similar. 
The former corresponds to a fitting term in the following equation:
\begin{equation}\label{eq:fit}
\OO_f = \|\ZM - f_{\Theta}(\XM)\|^2_F,
\end{equation}
where $f_{\Theta}(\XM)\in \mathbb{R}^{|\EDG|\times z}$ represents a non-linear transformation features of the input edge attribute matrix $\XM$ using an MLP $f_{\Theta}(\cdot)$ parameterized by a learnable weight matrix $\Theta\in \mathbb{R}^{d\times z}$ (including a nonlinear activation function ReLU operation and dropout operation), while the latter is a graph structure-based regularization term $\OO_r$ defined by
\begin{equation}\label{eq:reg}
\begin{split}
\OO_r = & \frac{\beta}{2} \sum_{u\in \U}{\sum_{e_i,e_j\in \EDG_u}}{\frac{1}{|\EDG_u|}\cdot \|\ZM[e_i]-\ZM[e_j]\|^2} \\
& + \frac{1-\beta}{2} \sum_{v\in \V}{\sum_{e_i,e_j\in \EDG_v}}{\frac{1}{|\EDG_v|}\cdot \|\ZM[e_i]-\ZM[e_j]\|^2}.
\end{split}
\end{equation}
Intuitively, Eq. \eqref{eq:reg} forces representations $\ZM[e_i],\ZM[e_j]$ to be close in the Euclidean space if their corresponding edges $e_i,e_j$ are correlated to common nodes. $\frac{1}{|\EDG_u|}$ (resp. $\frac{1}{|\EDG_v|}$) is the weight used to reflect the importance of $e_i,e_j$ from the perspective of common node $u$ (resp. $v$). In particular, we use coefficients $\beta$ and $1-\beta$ to control the importance of edge pairs' shared nodes from $\U$ and $\V$ in constraining the distance between the representations, respectively. 
In sum, the objective of learning $\ZM$ can be formulated as follows:
\begin{equation}\label{eq:obj}
\min_{\ZM\in \mathbb{R}^{|\EDG|\times z}}{(1-\alpha)\cdot \OO_f + \alpha \cdot \OO_r},
\end{equation}
where hyper-parameter $\alpha$ is to balance the above-said two terms.

\subsection{Analysis of the Optimal Solution}\label{sec:sol}
\begin{lemma}\label{lem:closed-form-sol}
The closed-form solution to Eq. \eqref{eq:obj} is
\begin{equation}\label{eq:Z}
\ZM = (1-\alpha)\sum_{t=0}^{\infty}{\alpha^t \PM^{t}} \cdot f_{\Theta}(\XM),
\end{equation}
where $\PM$ is an edge-wise transition matrix defined by
\begin{equation}\label{eq:P}
\PM = \EM\DM^{-1}\EM^{\top},\ \EM=\sqrt{\beta}\cdot \EM_{\U} \mathbin\Vert  \sqrt{1-\beta}\cdot \EM_{\V}.
\end{equation}
\end{lemma}
Lemma \ref{lem:closed-form-sol} offers a simple yet elegant way (i.e., Eq. \eqref{eq:Z}) to calculate the optimal edge representations $\ZM$ to the optimization objective in Eq. \eqref{eq:obj}. However, Eq. \eqref{eq:Z} requires summing up an infinite series of matrix multiplications, which is infeasible in practice, especially for large EABGs. A remedy is to compute an approximate version $\ZM^{\prime}$ by summing up at most $T+1$ terms with a small integer $T$:
\begin{equation}\label{eq:Zprime}
\ZM^{\prime} = (1-\alpha)\sum_{t=0}^{T}{\alpha^t \PM^{t}} \cdot f_{\Theta}(\XM).
\end{equation}
In what follows, we theoretically show that such a truncation is not a favorable choice in EABGs.



\begin{table}
\centering
\caption{{Properties of $\PM$ ($\beta=0.5$)}}\label{tbl:mix}
\vspace{-3mm}
\begin{small}
\begin{tabular}{l|c|c|c|c}
\hline
{\bf Dataset} & {\bf $\sigma_2$ } & {\bf $\sigma_2^2$} & {\bf $\frac{1}{1-\sigma_2^2}$} & {\bf $\frac{1}{1-\alpha \sigma_k^2}$}\\ \hline
{\bf AMiner} & 0.9997 & 0.9994 & 1574.3930 & 1.8207 \\
{\bf OAG} & 0.9999 & 0.9997 & 3780.6051 & 1.8775 \\
\hline
\end{tabular}
\end{small}
\vspace{-2mm}
\end{table}

\begin{lemma}\label{lem:PPP}
Given $\PM$ in Eq. \eqref{eq:P}, $\PM=\beta\cdot \PM_\U+ (1-\beta)\cdot \PM_\V$.
\end{lemma}

First, by Lemma \ref{lem:PPP}, $\PM$ is a linear combination of $\PM_\U$ and $\PM_\V$. Recall that both $\PM_\U$ and $\PM_\V$ are non-negative doubly stochastic, which further connotes that $\PM$ is {\em non-negative doubly stochastic} and can be regarded as a reversible Markov chain. Let $T_{mix}$ be its mixing time.
Using its doubly stochastic property and the Convergence Theorem in \cite{levin2017markov,motwani1995randomized}, when $t>T_{mix}$, $\PM^t f_{\Theta}(\XM)$ converges to a stationary distribution $\boldsymbol{\Pi}$, wherein $\boldsymbol{\Pi}[e_i]$ is a constant vector, i.e., $\mathbf{1}\cdot \|f_{\Theta}(\XM)[e_i]\|_1$.
Thus, $\ZM$ in Eq. \eqref{eq:Z} can be broken down into two parts\footnote{For the interest of space, we defer all proofs to Appendix \ref{sec:proof}.}:
\begin{small}
\begin{equation}\label{eq:Z-2-parts}
\left((1-\alpha)\sum_{t=0}^{T_{mix}-1}{\alpha^t \PM^{t}} \cdot f_{\Theta}(\XM) \right) + \alpha^{T_{mix}}\boldsymbol{\Pi}.
\end{equation}
\end{small}
Intuitively, since each row in $\boldsymbol{\Pi}$ is a constant vector, $\alpha^{T_{mix}}\boldsymbol{\Pi}$ is not an informative representation matrix.
As such, Eq. \eqref{eq:Z-2-parts} implies that if we pick a large $T$ ($T\gg T_{mix}$) for $\ZM^{\prime}$, constant vectors $\alpha^{T_{mix}}\boldsymbol{\Pi}$ might jeopardize the representation quality of $\ZM^{\prime}$ in Eq. \eqref{eq:Zprime}, especially on graphs with small mixing times, resulting in degraded performance. On the other hand, a small $T$ ($T\ll T_{mix}-1$) for $\ZM^{\prime}$ fails to enable an adequate preservation of the topological semantics underlying the input EABGs.
\begin{lemma}\label{lem:mix}
Let $\sigma_2$ be the second largest singular value of $\EM\DM^{-1/2}$ (defined in Eq. \eqref{eq:P}), $T_{mix} \ge \frac{1}{1-\sigma_2^2}-1$.
\end{lemma}

Lemma \ref{lem:mix} provides a lower bound for $T_{mix}$, which is proportional to the inverse of $1-\sigma_2^2$. As per the empirical data on real EABGs (see Section \ref{sec:datasets}) from Table \ref{tbl:mix}, $\sigma_2$ is notably approaching 1, rendering $T_{mix}$ extremely large (over thousands), as a consequence of the unique characteristics of bipartite graph structures. We can conclude that the $\alpha^{T_{mix}}\boldsymbol{\Pi}$ part in $\ZM$ (Eq. \eqref{eq:Z-2-parts}) is insignificant. Additionally, based on Section 12.2 in \cite{levin2017markov}, given any integer $t$ and any edge $e_i\in \EDG$,
\begin{equation*}
Var_{\boldsymbol{\Pi}[e_i]}(\PM^tf_{\Theta}(\XM)[e_i]) \le \sigma^{4t}_2\cdot Var_{\boldsymbol{\Pi}[e_i]}(f_{\Theta}(\XM)[e_i]),
\end{equation*}
where $Var_{\boldsymbol{\Pi}[e_i]}$ stands for the variance computed w.r.t. the stationary distribution $\boldsymbol{\Pi}[e_i]$. Since $\sigma_2$ is almost 1, the above equation manifests that even for a very large $t$, the difference between $\PM^t f_{\Theta}(\XM)[e_i]$ and the stationary distribution $\boldsymbol{\Pi}[e_i]$ can be as significant as that of the input feature vector $f_{\Theta}(\XM)[e_i]$. That is to say, $\PM^t f_{\Theta}(\XM)[e_i]$ with large $t$ still encompasses rich and informative features, and thus, computing $\ZM$ via Eq. \eqref{eq:Zprime} leads to compromised representation quality.

\subsection{Factorized Feature Propagation}\label{sec:ffp}
However, it remains tenaciously challenging to calculate the edge representations $\ZM$ by Eq. \eqref{eq:Z}. To tackle this issue, we resort to a dimensionality reduction approach, dubbed as {\em factorized feature propagation} (FFP). The rudimentary idea behind \ffp is to construct an $|\EDG|\times k$ ($k\ll |\EDG|$) matrix $\QM$ such that $\QM\cdot\QM^{\top}\approx (1-\alpha)\sum_{t=0}^{\infty}{\alpha^t \PM^{t}}$ without explicitly materializing $(1-\alpha)\sum_{t=0}^{\infty}{\alpha^t \PM^{t}}$. As such, edge representations $\ZM$ in Eq. \eqref{eq:Z} can be approximated via
\begin{equation}\label{eq:Z-Q}
\ZM = \QM\cdot (\QM^{\top}f_{\Theta}(\XM)),
\end{equation}
which can be done in $O(|\EDG|\cdot kz)$ time.
To realize the above idea, \ffp first conducts a $k$-truncated SVD over $\EM\DM^{-\frac{1}{2}}$ to get its left singular vectors $\UM$ and the diagonal matrix $\SVM$ containing singular values. Then, we construct $\QM$ as
\begin{equation}\label{eq:Q}
\QM=\UM\cdot \sqrt{\frac{1}{1-\alpha\SVM^2}}.
\end{equation}

The underlying rationale is on the basis of $\PM=\EM\DM^{-\frac{1}{2}}\cdot (\EM\DM^{-\frac{1}{2}})^{\top}\approx \UM\SVM^2\UM^{\top}$. The result in Eq. \eqref{eq:Q} is a direct inference after proving all singular values of $\EM\DM^{-\frac{1}{2}}$ not greater than 1 in Lemma \ref{lem:sigma}:
\begin{equation*}
(1-\alpha)\sum_{t=0}^{\infty}{\alpha^t \PM^{t}} \approx \UM\cdot \sum_{t=0}^{\infty}{\alpha^{t}\cdot \SVM^{2t}} \cdot \UM^{\top}=\UM\cdot \frac{1}{1-\alpha\SVM^2} \cdot \UM^{\top}.
\end{equation*}




\header
{\bf Correctness Analysis.} Theorem \ref{lem:QQS} establishes the approximation accuracy guarantees of $\QM$. In practice, we usually set $k=256$, and thus, the total error sum $\frac{1}{1-\alpha \sigma_k^2}$ is roughly 2, as reported in Table \ref{tbl:mix}.
\begin{theorem}\label{lem:QQS}
Let $\QM$ be the $|\EDG|\times k$ matrix defined in Eq. \eqref{eq:Q} and $\sigma_k$ be the $k$-th largest singular value of $\EM\DM^{-\frac{1}{2}}$. Then, the following inequality holds
\begin{equation*}
\left\| \QM\QM^{\top} - (1-\alpha)\sum_{t=0}^{\infty}{\alpha^t \PM^{t}} \right\|_F \le \frac{1}{1-\alpha \sigma_k^2}.
\end{equation*}
In particular, when $k=|\EDG|$,
\begin{equation*}
\QM\QM^{\top} = (1-\alpha)\sum_{t=0}^{\infty}{\alpha^t \PM^{t}}.
\end{equation*}
\end{theorem}

\header
{\bf Complexity Analysis.} As remarked earlier, Eq. \eqref{eq:Z-Q} consumes $O(|\EDG|\cdot kz)$ time. According to \cite{halko2011finding}, the $k$-truncated SVD of sparse matrix $\EM\DM^{-\frac{1}{2}}$ (comprising $2|\EDG|$ elements) takes $O(|\EDG|\cdot k^2+k^3)$ time when the randomized algorithm is employed. Overall, the total computational cost incurred by \ffp is bounded by $O(|\EDG| k\cdot (k+z))$.





\subsection{Model Training}\label{sec:loss}
In this work, we mainly focus on the semi-supervised edge classification task. The edge representations $\ZM$ output by \ffp are subsequently fed into an MLP network $f_{\Omega}(\cdot)$ to yield the edge classification result:
\begin{equation}\label{eq:f-omega}
\YM = \text{sigmoid}(f_{\Omega}(\ZM)) \in \mathbb{R}^{|\EDG|\times |C|}
\end{equation}
where $|C|$ is the number of classes and $f_{\Omega}$ is parameterized by a learnable weight matrix $\Omega\in \mathbb{R}^{k \times |C|}$, followed by a nonlinear activation function ReLU operation and a dropout operation. In sum, the trainable parameters of \algo are only the weight matrix $\Theta\in \mathbb{R}^{d\times z}$ in the feature transformation layer $f_{\Theta}(\XM)$ and the weight matrix $\Omega \in \mathbb{R}^{k \times |C|}$ of the output layer in Eq. \eqref{eq:f-omega}.

Following common practice, we employ the cross-entropy loss with ground-truth edge labels to guide the model training:
\begin{align*}
\mathcal{L}=-\frac{1}{|\EDG_L|}\sum_{e_i\in \EDG_L}{\sum_{j\in C}} & \widehat{\YM}[e_i,j]\cdot \log{(\YM[e_i,j])} \\
+ & (1-\widehat{\YM}[e_i,j])\cdot \log{\left(1-\YM[e_i,j]\right)} ,
\end{align*}
where $\EDG_L$ denotes the set of labeled edges, $\widehat{\YM}$ consists of the ground-truth labels of edges ($\widehat{\YM}[e_i,j]=1$ if $e_i$ belongs to class $C_j$ and 0 otherwise), and $\YM[e_i,j]$ stands for the predicted probability of edge $e_i$ belonging to class $C_j$.

\section{\algo with Dual-View \ffp}\label{sec:solution-opt}
Recall that in Section \ref{sec:sol}, the edge-wise transition matrix $\PM$ can be equivalently converted into $\PM=\beta\cdot \PM_\U+ (1-\beta)\cdot \PM_\V$ (Lemma \ref{lem:PPP}). In turn, the edge representations $\ZM$ in Eq. \eqref{eq:Z} are essentially obtained through a linear combination of the features propagated between edges via their connections using two heterogeneous node sets $\U$ and $\V$ as intermediaries, which tends to yield sub-optimal representation effectiveness. 
Further, such a linear combination relies on a manually selected parameter $\beta$ to balance the importance of features w.r.t. these two views, which requires re-calculating the $k$-truncated SVD of $\EM\DM^{-\frac{1}{2}}$ (see Eq. \eqref{eq:P}) from scratch to create $\QM$ (Eq. \eqref{eq:Q}) once $\beta$ was changed, leading to significant computation effort.

To mitigate the foregoing issues, in \algo, we develop {\em dual-view factorized feature propagation} (referred to as \dvffp) for learning enhanced edge representations. The basic idea is to create two intermediate edge representations, $\ZM_\U$ and $\ZM_\V$, by utilizing the associations between edges from the views of $\U$ and $\V$ severally, and then coalesce them into the final edge representations $\ZM$.

In the sequel, Section \ref{sec:dv-ffp} elaborates on the details of \dvffp, followed by a theoretical analysis in Section \ref{sec:dv-ffp-als}

\subsection{Dual-View Factorized Feature Propagation}\label{sec:dv-ffp}






Akin to Eq. \eqref{eq:Z}, the goal of \dvffp is to generate edge representations $\ZM_\U$ and $\ZM_\V$ from the $\U$-wise and $\V$-wise views as follows:
\begin{equation}\label{eq:ZUZV}
\begin{split}
\ZM_\U=(1-\alpha)\sum_{t=0}^{\infty}{\alpha^{t}\PM^{t}_\U} \cdot f_{\Theta_\U}(\XM),\\ \ZM_\V=(1-\alpha)\sum_{t=0}^{\infty}{\alpha^{t} \PM^{t}_\V} \cdot f_{\Theta_\V}(\XM).
\end{split}
\end{equation}
In Eq. \eqref{eq:ZUZV}, $f_{\Theta_\U}(\XM)$ (resp. $f_{\Theta_\V}(\XM)$) corresponds to the initial edge features used for the generation of $\ZM_\U$ (resp. $\ZM_\V$), which is transformed from the input edge attribute vectors $\XM$ through an MLP network parameterized by weight matrix ${\Theta_\U}$ (resp. ${\Theta_\V}$).

In analogy to \ffp in Section \ref{sec:ffp}, \dvffp adopts a low-dimensional matrix approximation trick to approximate $(1-\alpha)\sum_{t=0}^{\infty}{\alpha^{t}\PM^{t}_\U}$ and $(1-\alpha)\sum_{t=0}^{\infty}{\alpha^{t}\PM^{t}_\V}$, while sidestepping the explicit construction of these two $|\EDG|\times |\EDG|$ dense matrices. Specifically, \dvffp first applies a $k$-truncated SVD over $\EM_\U\DM^{-\frac{1}{2}}_\U$ and $\EM_\U\DM^{-\frac{1}{2}}_\V$, respectively, to get the left singular vectors $\UM_\U$, singular values $\SVM_\U$, and their counterparts $\UM_\V$ and $\SVM_\V$. 
Let $\QM_\U=\UM_\U \sqrt{\frac{1}{1-\alpha\SVM^2_\U}}$ and $\QM_\V=\UM_\V \sqrt{\frac{1}{1-\alpha\SVM^2_\V}}$. Then, the $\U$-wise and $\V$-wise edge representations $\ZM_\U$ and $\ZM_\V$ can be computed by
\begin{equation}\label{eq:ZU-ZV-comp}
\ZM_\U = \QM_\U\cdot \left(\QM^{\top}_\U f_{\Theta_\U}(\XM)\right)\ \text{and}\ \ZM_\V = \QM_\V\cdot \left(\QM^{\top}_\V f_{\Theta_\V}(\XM)\right),
\end{equation}
respectively.
Afterwards, they are combined as the final edge representations $\ZM$ through
\begin{equation}\label{eq:Z-combine}
\ZM = f_{\text{combine}}(\gamma\cdot\ZM_\U, (1-\gamma)\cdot\ZM_\V),
\end{equation}
where $f_{\text{combine}}(\cdot,\cdot)$ is a combinator function, which can be a summation operator $+$, matrix concatenation operator $\mathbin\Vert$, or max operator, and $\gamma\in [0,1]$ is a hyper-parameter.





\subsection{Analysis}\label{sec:dv-ffp-als}
In the rest of this section, we theoretically analyze the optimization objective of learning $\ZM_\U$, $\ZM_\V$ as in Eq. \eqref{eq:ZUZV}, the approximation accuracy guarantees of $\QM_\U$ and $\QM_\V$, as well as the computational expense of \dvffp, respectively.

\header
{\bf Optimization Objective.}
Recall that $\PM=\beta\cdot \PM_\U+ (1-\beta)\cdot \PM_\V$ by Lemma \ref{lem:PPP}. If we set $\beta$ to $1$ and $0$, $\PM$ in Eq. \eqref{eq:P} turns into $\PM_\U$ and $\PM_\V$, respectively. Accordingly, $\ZM$ defined in Eq. \eqref{eq:Z} becomes $\ZM_\U$ and $\ZM_\V$, if we replace $f_\Theta(\XM)$ by $f_{\Theta_\U}(\XM)$ and $f_{\Theta_\V}(\XM)$, respectively. Since Lemma \ref{lem:closed-form-sol} indicates that $\ZM$ in Eq. \eqref{eq:Z} is the closed solution to the objective in Eq. \eqref{eq:obj}, when $\beta=1$ or $\beta=0$, $\ZM_\U$ and $\ZM_\V$ defined in Eq. \eqref{eq:ZUZV} are thus the closed form solutions to the problems that minimize the following objectives:
\begin{small}
\begin{equation*}
\begin{split}
(1-\alpha) \|\ZM_\U - f_{\Theta_\U}(\XM)\|^2_F + \frac{\alpha}{2} \sum_{u\in \U}{\sum_{e_i,e_j\in \EDG_u}}{\frac{1}{|\EDG_u|}\cdot \|\ZM_\U[e_i]-\ZM_\U[e_j]\|^2},\\
(1-\alpha) \|\ZM_\V - f_{\Theta_\V}(\XM)\|^2_F + \frac{\alpha}{2} \sum_{v\in \V}{\sum_{e_i,e_j\in \EDG_v}}{\frac{1}{|\EDG_v|}\cdot \|\ZM_\V[e_i]-\ZM_\V[e_j]\|^2},
\end{split}
\end{equation*}
\end{small}
respectively.

\header
{\bf Correctness.}
Since when we set $\beta=1$ (resp. $\beta=0$), $\PM$ and $\EM\DM^{-1/2}$ turn into $\PM_\U$ and $\EM_\U\DM^{-1/2}_\U$ (resp. $\PM_\V$ and $\EM_\V\DM^{-1/2}_\V$).
Let $\sigma_k(\U)$ and $\sigma_k(\V)$ be the $k$-th largest singular value of $\EM_\U\DM_\U^{-{1}/{2}}$ and $\EM_\V\DM_\V^{-{1}/{2}}$, respectively. Based on Theorem \ref{lem:QQS}, we can derive the following inequalities
\begin{small}
\begin{equation*}
\begin{split}
\left\| \QM_\U\QM_\U^{\top} - (1-\alpha)\sum_{t=0}^{\infty}{\alpha^t \PM^{t}_\U} \right\|_F \le \frac{1}{1-\alpha \sigma_k^2(\U)},\\
\left\| \QM_\V\QM_\V^{\top} - (1-\alpha)\sum_{t=0}^{\infty}{\alpha^t \PM^{t}_\V} \right\|_F \le \frac{1}{1-\alpha \sigma_k^2(\V)}.
\end{split}
\end{equation*}
\end{small}
In particular, when $k=|\EDG|$,
\begin{equation*}
\QM_\U\QM^{\top}_\U = (1-\alpha)\sum_{t=0}^{\infty}{\alpha^t \PM^{t}_\U}\ \text{and}\ \QM_\V\QM^{\top}_\V = (1-\alpha)\sum_{t=0}^{\infty}{\alpha^t \PM^{t}_\V}.
\end{equation*}

\header
{\bf Complexity.}
The computations of $\ZM_\U$ and $\ZM_\V$ in Eq. \eqref{eq:ZU-ZV-comp} need $O(|\EDG|kz)$ time, respectively. 
The randomized $k$-truncated SVD \cite{halko2011finding} of sparse matrices $\EM_\U\DM_\U^{-\frac{1}{2}}$ and $\EM_\V\DM_\V^{-\frac{1}{2}}$ requires $O(|\EDG|\cdot k^2+k^3)$ time. Therefore, \dvffp and \ffp have the same time complexity $O(|\EDG| k\cdot (k+z))$.

\section{Experiments}\label{sec:experiments}

\begin{table}
\centering
\caption{Edge-Attributed Bipartite Networks}\label{tbl:datasets}
\vspace{-2mm}
\begin{tabular}{l|c|c|c|c|c}
\hline
{\bf Name} & {\bf $|\EDG|$} & {\bf $|\U|$} & {\bf $|\V|$} & {\bf $d$} &  {\bf $|C|$} \\ \hline
Amazon &  359,425 & 25,939 & 14,061 & 768 & 3  \\
AMiner &  54,465 & 39,358 & 641 & 768 & 10  \\
DBLP &  243,960 & 33,503 & 6497 & 768 & 10  \\
Google &  564,831 & 32,788 & 7212 & 768 & 3  \\
MAG &  50,443  & 38,990 & 1,010 & 768 & 10  \\ 
 \hline
 \end{tabular}
 \end{table}

\begin{table*}[!t]
\centering
\caption{Classification Performance (the higher the better).}\label{tbl:classification-perf}
\vspace{-2ex}
\begin{tabular}{|c|cc|cc|cc|cc|cc|} 
\hline
\multirow{2}{*}{\bf Method} & \multicolumn{2}{c|}{\bf Amazon} & \multicolumn{2}{c|}{\bf AMiner} & \multicolumn{2}{c|}{\bf DBLP} & \multicolumn{2}{c|}{\bf Google} &  \multicolumn{2}{c|}{\bf MAG}  \\ 
\cline{2-11}
                        & AP     & AUC                & AP     & AUC                & AP     & AUC              & AP     & AUC                & AP     & AUC                           \\ 
\hline
GCN                     & 0.6515 & 0.8691             & 0.0966 & 0.9254             & 0.5835 & 0.8979           & 0.5396 & 0.8122                       & 0.7504 & 0.9617           \\
GraphSAGE               & 0.6927 & 0.8874             & 0.1398 & 0.9485             & \underline{0.6785} & \underline{0.9254}           & \underline{0.5789} & \cellcolor{gray!20}\underline{0.8250}                       & \underline{0.7998} & \underline{0.9702}           \\
SGC                     & 0.5721 & 0.8203             & 0.0360 & 0.8468             & 0.4753 & 0.8576           & 0.4728 & 0.7563                       & 0.6838 & 0.9470           \\
DGI                     & 0.3879 & 0.6094             & 0.0046 & 0.5024             & 0.2785 & 0.6813           & 0.3465 & 0.5336                        & 0.4278 & 0.8315           \\
GAT                     & 0.6809 & 0.8810             & 0.1197 & 0.9077             & 0.5936 & 0.8988           & 0.5040 & 0.6928                        & 0.7517 & 0.9614           \\
GATv2                   & 0.6871 & 0.8860             & 0.1356 & 0.9206             & 0.6462 & 0.9156           & 0.5313 & 0.7804                        & 0.7673 & 0.9635           \\
FC                      &  \underline{0.7030} &  \underline{0.8905}             & \underline{0.1877} & \underline{0.9607}             & 0.6234 & 0.8922           & 0.5585 & 0.7890                       & 0.7401 & 0.9564           \\

GEBE                    & 0.4751 & 0.7158             & 0.1013 & 0.8739             & 0.4164 & 0.8116           & 0.3956 & 0.6508                       & 0.6555 & 0.9317           \\
AttrE2Vec               & 0.3334 & 0.4991             & 0.0080 & 0.5966             & 0.1716 & 0.5390           & 0.3331 & 0.4976                        & 0.1480 & 0.5463           \\ 
\hline
\algo (\ffp)        &  0.6946 & 0.8875 & \cellcolor{gray!45}0.5209 & \cellcolor{gray!45}0.9754 & 0.7069 & 0.9325 & 0.5787 & 0.8151 & 0.9047 & \cellcolor{gray!20}0.9869           \\
\algo (\dvffp)-sum     & 0.7059 & 0.8941 & 0.5062 & \cellcolor{gray!20}0.9684 & 0.7160 & 0.9319 & 0.5886 & 0.8202 & 0.9083 & 0.9867 \\
\algo (\dvffp)-max     & \cellcolor{gray!45}0.7093 & \cellcolor{gray!45}0.8965 & 0.3740 & 0.9589 & \cellcolor{gray!45}0.7267 & \cellcolor{gray!45}0.9361 & \cellcolor{gray!45}0.5968 & \cellcolor{gray!45}0.8287 & \cellcolor{gray!45}0.9195 & \cellcolor{gray!45}0.9888 \\
\algo (\dvffp)-concat     & \cellcolor{gray!20}0.7064 & \cellcolor{gray!20}0.8944 & \cellcolor{gray!20}0.5081 & 0.9660 & \cellcolor{gray!20}0.7187 & \cellcolor{gray!20}0.9327 & 0.5931 & 0.8248 & \cellcolor{gray!20}0.9128 & \cellcolor{gray!20}0.9869 \\
\hline
\end{tabular}
\end{table*}

In this section, we empirically study the effectiveness of our proposed \algo models on real-world datasets in terms of edge classification. All experiments are conducted on a Linux machine powered by 4 AMD EPYC 7313 CPUs with 500GB RAM, and 1 NVIDIA RTX A5000 GPU with 24GB memory. The code and all datasets are available at \url{https://github.com/wanghewen/EAGLE} for reproducibility.

\subsection{Baselines and Hyperparameters}

We compare our proposed solutions against 9 competitors in terms of edge classification accuracy. The first category of baseline models consists of node-wise representation learning methods, including GCN~\cite{Kipf2016}, GraphSAGE~\cite{Hamilton2017}, SGC~\cite{Wu2019}, DGI~\cite{velivckovic2018deep}, GAT~\cite{Velickovivelickovic}, and GATv2~\cite{brody2022how}. We initialize the embeddings of edge endpoints as the mean average of their connected edge attributes. Then, we apply these node-wise representation learning methods to update the node embeddings for the edge endpoints. Finally, we concatenate the embeddings of edge endpoints along with edge attributes to generate the corresponding edge embeddings.
The second category of baseline models consists of edge-wise representation learning methods, including GEBE~\cite{yang2022scalable} and AttrE2Vec~\cite{bielak2022attre2vec}. Additionally, we include a fully connected neural network (FC) to transform edge attributes without considering any network structure information.

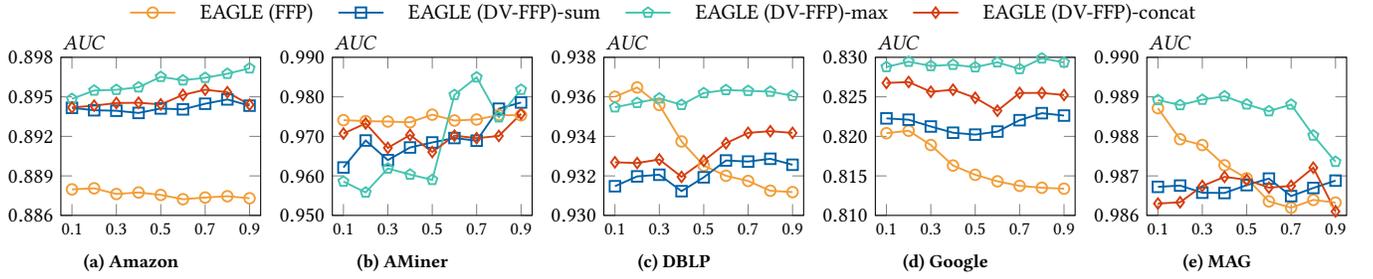
\begin{figure*}[!t]
\centering
\begin{small}
\begin{tikzpicture}
    \begin{customlegend}[legend columns=4,
        legend entries={\algo (\ffp),\algo (\dvffp)-sum,\algo (\dvffp)-max,\algo (\dvffp)-concat},
        legend style={at={(0.45,1.15)},anchor=north,draw=none,font=\small,column sep=0.25cm}]
    \addlegendimage{line width=0.25mm,mark size=2pt,mark=o, color=O1}
    \addlegendimage{line width=0.25mm,mark size=2pt,mark=square, color=B2}
    \addlegendimage{line width=0.25mm,mark size=2pt,mark=pentagon, color=B6}
    \addlegendimage{line width=0.25mm,mark size=2pt,mark=diamond, color=O3}
    \end{customlegend}
\end{tikzpicture}
\\[-\lineskip]
\vspace{-4mm}
\subfloat[{Amazon}]{
\begin{tikzpicture}[scale=1]\label{subfig:Amazon}
    \begin{axis}[
        height=\columnwidth/2.3,
        width=\columnwidth/2.0,
        ylabel={\em AUC},
        xmin=0.5, xmax=9.5,
        ymin=0.886, ymax=0.898,
        xtick={1,3,5,7,9},
        xticklabel style = {font=\footnotesize},
        xticklabels={0.1,0.3,0.5,0.7,0.9},
        ytick={0.886,0.889,0.892,0.895,0.898},
        scaled y ticks = false,
        yticklabel style={/pgf/number format/fixed zerofill,/pgf/number format/precision=3},
        every axis y label/.style={at={(current axis.north west)},right=3mm,above=0mm},
        label style={font=\small},
        tick label style={font=\small},
    ]
        
    \addplot[line width=0.25mm,mark size=2pt,mark=o, color=O1] 
        plot coordinates {
(1,	0.88797515	)
(2,	0.888056264	)
(3,	0.887623586	)
(4,	0.887742534	)
(5,	0.887553433	)
(6,	0.887227595	)
(7,	0.887355194	)
(8,	0.887456487	)
(9,	0.887301247	)
};
    \addplot[line width=0.25mm,mark size=2pt,mark=square, color=B2] 
        plot coordinates {
(1,	0.894165618	)
(2,	0.893981246	)
(3,	0.893944865	)
(4,	0.893777961	)
(5,	0.894116297	)
(6,	0.894047922	)
(7,	0.89447681	)
(8,	0.894803636	)
(9,	0.894329931	)
};
    \addplot[line width=0.25mm,mark size=2pt,mark=pentagon, color=B6] 
        plot coordinates {
(1,	0.894847536	)
(2,	0.895480559	)
(3,	0.895536113	)
(4,	0.895727936	)
(5,	0.896521065	)
(6,	0.896256141	)
(7,	0.896446683	)
(8,	0.896753706	)
(9,	0.89716433	)
};
    \addplot[line width=0.25mm,mark size=2pt,mark=diamond, color=O3] 
        plot coordinates {
(1,	0.894190708	)
(2,	0.894331219	)
(3,	0.894530645	)
(4,	0.894552151	)
(5,	0.894426545	)
(6,	0.895145109	)
(7,	0.895541958	)
(8,	0.895351221	)
(9,	0.894402576	)
};

    \end{axis}
\end{tikzpicture}\hspace{0mm}%
}%
\subfloat[{AMiner}]{
\begin{tikzpicture}[scale=1]\label{subfig:AMiner}
    \begin{axis}[
        height=\columnwidth/2.3,
        width=\columnwidth/2.0,
        ylabel={\em AUC},
        xmin=0.5, xmax=9.5,
        ymin=0.95, ymax=0.99,
        xtick={1,3,5,7,9},
        xticklabel style = {font=\footnotesize},
        xticklabels={0.1,0.3,0.5,0.7,0.9},
        ytick={0.95,0.96,0.97,0.98,0.99},
        scaled y ticks = false,
        yticklabel style={/pgf/number format/fixed zerofill,/pgf/number format/precision=3},
        every axis y label/.style={at={(current axis.north west)},right=3mm,above=0mm},
        label style={font=\small},
        tick label style={font=\small},
    ]
        
    \addplot[line width=0.25mm,mark size=2pt,mark=o, color=O1] 
        plot coordinates {
(1,	0.974102891	)
(2,	0.973854614	)
(3,	0.973748442	)
(4,	0.973565788	)
(5,	0.97548689	)
(6,	0.9740103	)
(7,	0.974261848	)
(8,	0.975399401	)
(9,	0.975328847	)
};
    \addplot[line width=0.25mm,mark size=2pt,mark=square, color=B2] 
        plot coordinates {
(1,	0.96212247	)
(2,	0.968951646	)
(3,	0.96399075	)
(4,	0.967183937	)
(5,	0.968455675	)
(6,	0.969576875	)
(7,	0.968900645	)
(8,	0.976998035	)
(9,	0.97857109	)
};
    \addplot[line width=0.25mm,mark size=2pt,mark=pentagon, color=B6] 
        plot coordinates {
(1,	0.958566306	)
(2,	0.955805601	)
(3,	0.961876292	)
(4,	0.960341278	)
(5,	0.958992728	)
(6,	0.980496212	)
(7,	0.984953422	)
(8,	0.974827742	)
(9,	0.981816712	)
};
    \addplot[line width=0.25mm,mark size=2pt,mark=diamond, color=O3] 
        plot coordinates {
(1,	0.970714267	)
(2,	0.973283238	)
(3,	0.96708018	)
(4,	0.970357435	)
(5,	0.9660388	)
(6,	0.970185454	)
(7,	0.969534865	)
(8,	0.970106273	)
(9,	0.975704959	)
};
    \end{axis}
\end{tikzpicture}\hspace{0mm}%
}%
\subfloat[{DBLP}]{
\begin{tikzpicture}[scale=1]\label{subfig:DBLP}
    \begin{axis}[
        height=\columnwidth/2.3,
        width=\columnwidth/2.0,
        ylabel={\em AUC},
        xmin=0.5, xmax=9.5,
        ymin=0.930, ymax=0.938,
        xtick={1,3,5,7,9},
        xticklabel style = {font=\footnotesize},
        xticklabels={0.1,0.3,0.5,0.7,0.9},
        ytick={0.930,0.932,0.934,0.936,0.938},
        scaled y ticks = false,
        yticklabel style={/pgf/number format/fixed zerofill,/pgf/number format/precision=3},
        every axis y label/.style={at={(current axis.north west)},right=3mm,above=0mm},
        label style={font=\small},
        tick label style={font=\small},
    ]
        
    \addplot[line width=0.25mm,mark size=2pt,mark=o, color=O1] 
        plot coordinates {
(1,	0.93600524	)
(2,	0.936475448	)
(3,	0.935587829	)
(4,	0.933741586	)
(5,	0.932515943	)
(6,	0.93200449	)
(7,	0.931742961	)
(8,	0.931247969	)
(9,	0.931179392	)
};
    \addplot[line width=0.25mm,mark size=2pt,mark=square, color=B2] 
        plot coordinates {
(1,	0.93147634	)
(2,	0.931973655	)
(3,	0.932056894	)
(4,	0.931226347	)
(5,	0.931927485	)
(6,	0.932778267	)
(7,	0.932724437	)
(8,	0.932864228	)
(9,	0.932562448	)
};
    \addplot[line width=0.25mm,mark size=2pt,mark=pentagon, color=B6] 
        plot coordinates {
(1,	0.93547266	)
(2,	0.935696832	)
(3,	0.93592431	)
(4,	0.935594772	)
(5,	0.936193671	)
(6,	0.936347266	)
(7,	0.936304491	)
(8,	0.936263884	)
(9,	0.936053998	)
};
    \addplot[line width=0.25mm,mark size=2pt,mark=diamond, color=O3] 
        plot coordinates {
(1,	0.932686064	)
(2,	0.932650063	)
(3,	0.932829728	)
(4,	0.931950698	)
(5,	0.932760543	)
(6,	0.933650529	)
(7,	0.93417656	)
(8,	0.934260947	)
(9,	0.934175271	)
};
    \end{axis}
\end{tikzpicture}\hspace{0mm}%
}%
\subfloat[{Google}]{
\begin{tikzpicture}[scale=1]\label{subfig:Google}
    \begin{axis}[
        height=\columnwidth/2.3,
        width=\columnwidth/2.0,
        ylabel={\em AUC},
        xmin=0.5, xmax=9.5,
        ymin=0.81, ymax=0.83,
        xtick={1,3,5,7,9},
        xticklabel style = {font=\footnotesize},
        xticklabels={0.1,0.3,0.5,0.7,0.9},
        ytick={0.81,0.815,0.82,0.825,0.83},
        scaled y ticks = false,
        yticklabel style={/pgf/number format/fixed zerofill,/pgf/number format/precision=3},
        every axis y label/.style={at={(current axis.north west)},right=3mm,above=0mm},
        label style={font=\small},
        tick label style={font=\small},
    ]
        
    \addplot[line width=0.25mm,mark size=2pt,mark=o, color=O1] 
        plot coordinates {
(1,	0.82040079	)
(2,	0.820712382	)
(3,	0.818904721	)
(4,	0.816306789	)
(5,	0.815134518	)
(6,	0.81431304	)
(7,	0.813740359	)
(8,	0.813499696	)
(9,	0.813379067	)
};
    \addplot[line width=0.25mm,mark size=2pt,mark=square, color=B2] 
        plot coordinates {
(1,	0.822275104	)
(2,	0.822106355	)
(3,	0.821246864	)
(4,	0.820459987	)
(5,	0.820220357	)
(6,	0.82060546	)
(7,	0.822044764	)
(8,	0.822934213	)
(9,	0.822630457	)
};
    \addplot[line width=0.25mm,mark size=2pt,mark=pentagon, color=B6] 
        plot coordinates {
(1,	0.828774653	)
(2,	0.829452242	)
(3,	0.828905697	)
(4,	0.829064451	)
(5,	0.828747509	)
(6,	0.8293741	)
(7,	0.828523874	)
(8,	0.829867203	)
(9,	0.829352285	)
};
    \addplot[line width=0.25mm,mark size=2pt,mark=diamond, color=O3] 
        plot coordinates {
(1,	0.826755671	)
(2,	0.826873508	)
(3,	0.825641453	)
(4,	0.825897403	)
(5,	0.824874116	)
(6,	0.82324476	)
(7,	0.825490542	)
(8,	0.82549844	)
(9,	0.82522806	)
};
    \end{axis}
\end{tikzpicture}\hspace{0mm}%
}%
\subfloat[{MAG}]{
\begin{tikzpicture}[scale=1]\label{subfig:MAG}
    \begin{axis}[
        height=\columnwidth/2.3,
        width=\columnwidth/2.0,
        ylabel={\em AUC},
        xmin=0.5, xmax=9.5,
        ymin=0.986, ymax=0.990,
        xtick={1,3,5,7,9},
        xticklabel style = {font=\footnotesize},
        xticklabels={0.1,0.3,0.5,0.7,0.9},
        ytick={0.986,0.987,0.988,0.989,0.990},
        scaled y ticks = false,
        yticklabel style={/pgf/number format/fixed zerofill,/pgf/number format/precision=3},
        every axis y label/.style={at={(current axis.north west)},right=3mm,above=0mm},
        label style={font=\small},
        tick label style={font=\small},
    ]
        
    \addplot[line width=0.25mm,mark size=2pt,mark=o, color=O1] 
        plot coordinates {
(1,	0.988713539	)
(2,	0.987926124	)
(3,	0.987779403	)
(4,	0.987267063	)
(5,	0.986937501	)
(6,	0.986353769	)
(7,	0.986194516	)
(8,	0.986392141	)
(9,	0.986324085	)
};
    \addplot[line width=0.25mm,mark size=2pt,mark=square, color=B2] 
        plot coordinates {
(1,	0.986722089	)
(2,	0.986754229	)
(3,	0.986574432	)
(4,	0.986566624	)
(5,	0.986766096	)
(6,	0.986931327	)
(7,	0.986488243	)
(8,	0.986694073	)
(9,	0.986879898	)
};
    \addplot[line width=0.25mm,mark size=2pt,mark=pentagon, color=B6] 
        plot coordinates {
(1,	0.988918922	)
(2,	0.988792592	)
(3,	0.988928401	)
(4,	0.989015152	)
(5,	0.988810156	)
(6,	0.988637474	)
(7,	0.988805365	)
(8,	0.988022881	)
(9,	0.987356936	)
};
    \addplot[line width=0.25mm,mark size=2pt,mark=diamond, color=O3] 
        plot coordinates {
(1,	0.986300215	)
(2,	0.986331783	)
(3,	0.986744425	)
(4,	0.986974455	)
(5,	0.986901651	)
(6,	0.986702055	)
(7,	0.986746528	)
(8,	0.987206694	)
(9,	0.986100375	)
};
    \end{axis}
\end{tikzpicture}\hspace{0mm}%
}%
        
\end{small}
\vspace{-3mm}
\caption{Varying $\beta$ in \algo (\ffp) and $\gamma$ in \algo (\dvffp).} \label{fig:param}
\vspace{-2mm}
\end{figure*}

For DGI, GEBE, and AttrE2Vec, we collect the source codes from the respective authors and adopt the parameter settings suggested in their papers to generate edge representations before feeding them to MLPs (multi-layer perceptrons) for classification. For GCN, GraphSAGE, SGC, GAT, and GATv2,  we utilize the standard implementations provided in the well-known DGL\footnote{\url{ https://www.dgl.ai}} library and follow a three-layer neural network architecture, including two GNN layers and one linear layer, with ReLU as activation functions between layers. Besides, we set the dropout rate to 0.5 and the maximum number of training epochs to 300, and employ the Adam optimizer~\cite{kingma2017adam} for optimization with a learning rate of 0.001. All the methods are implemented in Python. 
In our solutions (i.e., \algo (\ffp) and \algo (\dvffp)), unless otherwise specified, we set the hyperparameter $\alpha$ and $\beta$ to be 0.5, $\gamma$ in Eq. \eqref{eq:Z-combine} to be 0.5, and dimension $k$ to be 256. The edge representations are then input to MLP classifiers to obtain the final edge labels. 
We report the AP/AUC on the test datasets using the model selected with the best AUC achieved on the cross-validation datasets. 

\subsection{Datasets}\label{sec:datasets}

We use 5 real-world bipartite network datasets in the experiments. The Amazon
dataset~\cite{Ni2019} contains user reviews for movies and TV shows, where the edges represent the reviews written by users on the products, which are associated with labels representing users' ratings on these products.
The Google 
dataset~\cite{Li2022,Yan2023} contains review information of business entities on Google Maps in Hawaii, United States, where the edges are reviews written by users on the business entity IDs. Similarly, the edge labels represent users' ratings on the business entities. 
AMiner~\cite{Tang2008}, MAG~\cite{sinha2015overview,yang2021effective} and DBLP~\cite{Tang:07ICDM} datasets are 3 citation networks, in which nodes represent scholars and their publication venues of a paper. The edges represent the paper abstracts written for that paper. For AMiner, edge labels correspond to the keywords for the papers. For DBLP and MAG, edge labels correspond to the field of study for the papers. We select the most frequent 10 labels as targets to be predicted.
To obtain initial edge features from text for these datasets, we apply the Sentence-BERT~\cite{Reimers2019} model to encode text into 768-dimensional vectors. 
For each dataset, we use breadth-first search (BFS) to sample a smaller subset. Then we randomly split all edges into training, cross-validation and test sets with an $8:1:1$ ratio. The properties and scales of the datasets used in our experiments are summarized in Table \ref{tbl:datasets}. 
 
 

 
 

 
 



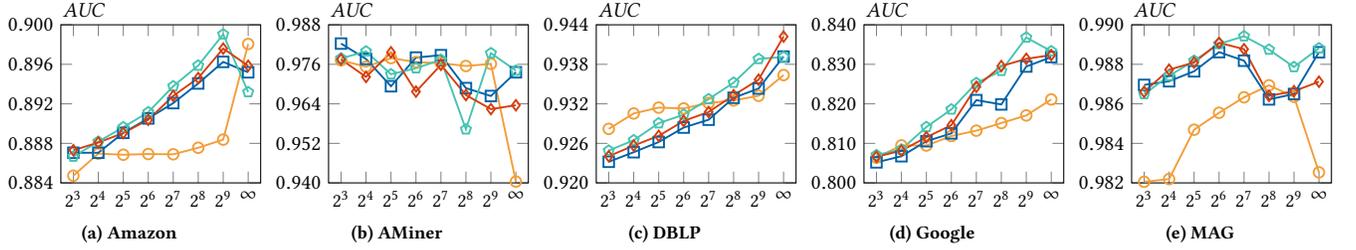
\begin{figure*}[!t]
\centering
\begin{small}
\vspace{-3mm}
\subfloat[{Amazon}]{
\begin{tikzpicture}[scale=1]\label{subfig:Amazon2}
    \begin{axis}[
        height=\columnwidth/2.3,
        width=\columnwidth/2.0,
        ylabel={\em AUC},
        xmin=0.5, xmax=8.5,
        ymin=0.884, ymax=0.9,
        xtick={1,2,3,4,5,6,7,8},
        xticklabel style = {font=\footnotesize},
        xticklabels={$2^3$,$2^4$,$2^5$,$2^6$,$2^7$,$2^8$,$2^9$,$\infty$},
        ytick={0.884,0.888,0.892,0.896,0.90},
        scaled y ticks = false,
        yticklabel style={/pgf/number format/fixed zerofill,/pgf/number format/precision=3},
        every axis y label/.style={at={(current axis.north west)},right=3mm,above=0mm},
        label style={font=\small},
        tick label style={font=\small},
    ]
        
    \addplot[line width=0.25mm,mark size=2pt,mark=o, color=O1] 
        plot coordinates {
(1,	0.884729818	)
(2,	0.886962499	)
(3,	0.886849424	)
(4,	0.886921845	)
(5,	0.886894439	)
(6,	0.887553433	)
(7,	0.888379073	)
(8,	0.898058207	)
};
    \addplot[line width=0.25mm,mark size=2pt,mark=square, color=B2] 
        plot coordinates {
(1,	0.887047721	)
(2,	0.887049462	)
(3,	0.889067339	)
(4,	0.890521715	)
(5,	0.892075445	)
(6,	0.894062833	)
(7,	0.896245972	)
(8,	0.89519188	)
};
    \addplot[line width=0.25mm,mark size=2pt,mark=pentagon, color=B6] 
        plot coordinates {
(1,	0.886663783	)
(2,	0.888172278	)
(3,	0.889638048	)
(4,	0.891156026	)
(5,	0.893782921	)
(6,	0.895874906	)
(7,	0.89903692	)
(8,	0.893190527	)
};
    \addplot[line width=0.25mm,mark size=2pt,mark=diamond, color=O3] 
        plot coordinates {
(1,	0.887303793	)
(2,	0.888083491	)
(3,	0.889046574	)
(4,	0.890392017	)
(5,	0.892803261	)
(6,	0.8946148	)
(7,	0.89759123	)
(8,	0.895852339	)
};

    \end{axis}
\end{tikzpicture}\hspace{0mm}%
}%
\subfloat[{AMiner}]{
\begin{tikzpicture}[scale=1]\label{subfig:AMiner2}
    \begin{axis}[
        height=\columnwidth/2.3,
        width=\columnwidth/2.0,
        ylabel={\em AUC},
        xmin=0.5, xmax=8.5,
        ymin=0.94, ymax=0.988,
        xtick={1,2,3,4,5,6,7,8},
        xticklabel style = {font=\footnotesize},
        xticklabels={$2^3$,$2^4$,$2^5$,$2^6$,$2^7$,$2^8$,$2^9$,$\infty$},
        ytick={0.94,0.952,0.964,0.976,0.988},
        scaled y ticks = false,
        yticklabel style={/pgf/number format/fixed zerofill,/pgf/number format/precision=3},
        every axis y label/.style={at={(current axis.north west)},right=3mm,above=0mm},
        label style={font=\small},
        tick label style={font=\small},
    ]
        
    \addplot[line width=0.25mm,mark size=2pt,mark=o, color=O1] 
        plot coordinates {
(1,	0.977229542	)
(2,	0.97525861	)
(3,	0.977887953	)
(4,	0.976676759	)
(5,	0.976364497	)
(6,	0.97548689	)
(7,	0.976144741	)
(8,	0.940395546	)
};
    \addplot[line width=0.25mm,mark size=2pt,mark=square, color=B2] 
        plot coordinates {
(1,	0.982298076	)
(2,	0.977627571	)
(3,	0.969345284	)
(4,	0.978051531	)
(5,	0.978773473	)
(6,	0.968777013	)
(7,	0.966380046	)
(8,	0.973579823	)
};
    \addplot[line width=0.25mm,mark size=2pt,mark=pentagon, color=B6] 
        plot coordinates {
(1,	0.977790503	)
(2,	0.979923378	)
(3,	0.972979315	)
(4,	0.974817184	)
(5,	0.977373802	)
(6,	0.956260586	)
(7,	0.979389003	)
(8,	0.97409043	)
};
    \addplot[line width=0.25mm,mark size=2pt,mark=diamond, color=O3] 
        plot coordinates {
(1,	0.977495291	)
(2,	0.972175781	)
(3,	0.979641045	)
(4,	0.967718401	)
(5,	0.975845916	)
(6,	0.966661064	)
(7,	0.962406873	)
(8,	0.963567251	)
};
    \end{axis}
\end{tikzpicture}\hspace{0mm}%
}%
\subfloat[{DBLP}]{
\begin{tikzpicture}[scale=1]\label{subfig:DBLP2}
    \begin{axis}[
        height=\columnwidth/2.3,
        width=\columnwidth/2.0,
        ylabel={\em AUC},
        xmin=0.5, xmax=8.5,
        ymin=0.92, ymax=0.944,
        xtick={1,2,3,4,5,6,7,8},
        xticklabel style = {font=\footnotesize},
        xticklabels={$2^3$,$2^4$,$2^5$,$2^6$,$2^7$,$2^8$,$2^9$,$\infty$},
        ytick={0.92,0.926,0.932,0.938,0.944},
        scaled y ticks = false,
        yticklabel style={/pgf/number format/fixed zerofill,/pgf/number format/precision=3},
        every axis y label/.style={at={(current axis.north west)},right=3mm,above=0mm},
        label style={font=\small},
        tick label style={font=\small},
    ]
        
    \addplot[line width=0.25mm,mark size=2pt,mark=o, color=O1] 
        plot coordinates {
(1,	0.928199597	)
(2,	0.930537591	)
(3,	0.931427994	)
(4,	0.931303236	)
(5,	0.932131983	)
(6,	0.932515943	)
(7,	0.933220781	)
(8,	0.936368252	)
};
    \addplot[line width=0.25mm,mark size=2pt,mark=square, color=B2] 
        plot coordinates {
(1,	0.923203393	)
(2,	0.924657987	)
(3,	0.926198908	)
(4,	0.928399027	)
(5,	0.929591596	)
(6,	0.932920039	)
(7,	0.934322592	)
(8,	0.939177601	)
};
    \addplot[line width=0.25mm,mark size=2pt,mark=pentagon, color=B6] 
        plot coordinates {
(1,	0.92490113	)
(2,	0.926478421	)
(3,	0.929058223	)
(4,	0.930360616	)
(5,	0.932716969	)
(6,	0.935216691	)
(7,	0.938819757	)
(8,	0.939114093	)
};
    \addplot[line width=0.25mm,mark size=2pt,mark=diamond, color=O3] 
        plot coordinates {
(1,	0.924045301	)
(2,	0.92566824	)
(3,	0.927137041	)
(4,	0.929403649	)
(5,	0.930761095	)
(6,	0.933310616	)
(7,	0.935670497	)
(8,	0.942209433	)
};
    \end{axis}
\end{tikzpicture}\hspace{0mm}%
}%
\subfloat[{Google}]{
\begin{tikzpicture}[scale=1]\label{subfig:Google2}
    \begin{axis}[
        height=\columnwidth/2.3,
        width=\columnwidth/2.0,
        ylabel={\em AUC},
        xmin=0.5, xmax=8.5,
        ymin=0.80, ymax=0.84,
        xtick={1,2,3,4,5,6,7,8},
        xticklabel style = {font=\footnotesize},
        xticklabels={$2^3$,$2^4$,$2^5$,$2^6$,$2^7$,$2^8$,$2^9$,$\infty$},
        ytick={0.8,0.81,0.82,0.83,0.84},
        scaled y ticks = false,
        yticklabel style={/pgf/number format/fixed zerofill,/pgf/number format/precision=3},
        every axis y label/.style={at={(current axis.north west)},right=3mm,above=0mm},
        label style={font=\small},
        tick label style={font=\small},
    ]
        
    \addplot[line width=0.25mm,mark size=2pt,mark=o, color=O1] 
        plot coordinates {
(1,	0.806383918	)
(2,	0.809493204	)
(3,	0.809448842	)
(4,	0.81184811	)
(5,	0.813212982	)
(6,	0.815134518	)
(7,	0.817069714	)
(8,	0.821109588	)
};
    \addplot[line width=0.25mm,mark size=2pt,mark=square, color=B2] 
        plot coordinates {
(1,	0.805192165	)
(2,	0.806744572	)
(3,	0.810543337	)
(4,	0.812524048	)
(5,	0.820938386	)
(6,	0.819853536	)
(7,	0.829442025	)
(8,	0.831825648	)
};
    \addplot[line width=0.25mm,mark size=2pt,mark=pentagon, color=B6] 
        plot coordinates {
(1,	0.807029148	)
(2,	0.808454518	)
(3,	0.814141838	)
(4,	0.818591233	)
(5,	0.825278668	)
(6,	0.828357097	)
(7,	0.83678841	)
(8,	0.83332223	)
};
    \addplot[line width=0.25mm,mark size=2pt,mark=diamond, color=O3] 
        plot coordinates {
(1,	0.806606072	)
(2,	0.808147803	)
(3,	0.811559224	)
(4,	0.814591818	)
(5,	0.824326945	)
(6,	0.829554708	)
(7,	0.831393444	)
(8,	0.832331982	) 
};
    \end{axis}
\end{tikzpicture}\hspace{0mm}%
}%
\subfloat[{MAG}]{
\begin{tikzpicture}[scale=1]\label{subfig:MAG2}
    \begin{axis}[
        height=\columnwidth/2.3,
        width=\columnwidth/2.0,
        ylabel={\em AUC},
        xmin=0.5, xmax=8.5,
        ymin=0.982, ymax=0.99,
        xtick={1,2,3,4,5,6,7,8},
        xticklabel style = {font=\footnotesize},
        xticklabels={$2^3$,$2^4$,$2^5$,$2^6$,$2^7$,$2^8$,$2^9$,$\infty$},
        ytick={0.982,0.984,0.986,0.988,0.99},
        scaled y ticks = false,
        yticklabel style={/pgf/number format/fixed zerofill,/pgf/number format/precision=3},
        every axis y label/.style={at={(current axis.north west)},right=3mm,above=0mm},
        label style={font=\small},
        tick label style={font=\small},
    ]
        
    \addplot[line width=0.25mm,mark size=2pt,mark=o, color=O1] 
        plot coordinates {
(1,	0.982055314	)
(2,	0.98219271	)
(3,	0.984690459	)
(4,	0.985547318	)
(5,	0.986332199	)
(6,	0.986937501	)
(7,	0.986358989	)
(8,	0.982533252	)
};
    \addplot[line width=0.25mm,mark size=2pt,mark=square, color=B2] 
        plot coordinates {
(1,	0.986957802	)
(2,	0.987143023	)
(3,	0.987646585	)
(4,	0.988617809	)
(5,	0.988182305	)
(6,	0.986224837	)
(7,	0.986497655	)
(8,	0.988610666	)
};
    \addplot[line width=0.25mm,mark size=2pt,mark=pentagon, color=B6] 
        plot coordinates {
(1,	0.986493127	)
(2,	0.987417616	)
(3,	0.988182105	)
(4,	0.989031469	)
(5,	0.989421901	)
(6,	0.988744497	)
(7,	0.987862924	)
(8,	0.988813793	)
};
    \addplot[line width=0.25mm,mark size=2pt,mark=diamond, color=O3] 
        plot coordinates {
(1,	0.986610411	)
(2,	0.987728918	)
(3,	0.988092036	)
(4,	0.989092404	)
(5,	0.988765432	)
(6,	0.986400547	)
(7,	0.986641257	)
(8,	0.987119819	)
};
    \end{axis}
\end{tikzpicture}\hspace{0mm}%
}%
        
\end{small}
\vspace{-3mm}
\caption{Varying $k$.} \label{fig:param2}
\vspace{-3mm}
\end{figure*}

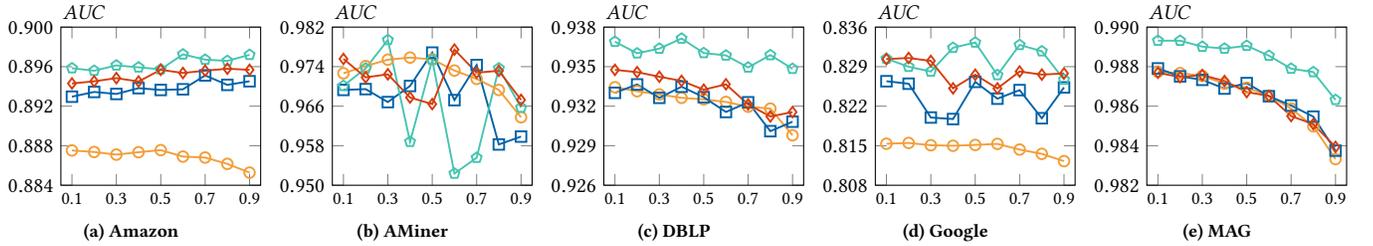
\begin{figure*}[!t]
\centering
\begin{small}
\vspace{-3mm}
\subfloat[{Amazon}]{
\begin{tikzpicture}[scale=1]\label{subfig:Amazon3}
    \begin{axis}[
        height=\columnwidth/2.3,
        width=\columnwidth/2.0,
        ylabel={\em AUC},
        xmin=0.5, xmax=9.5,
        ymin=0.884, ymax=0.9,
        xtick={1,3,5,7,9},
        xticklabel style = {font=\footnotesize},
        xticklabels={0.1,0.3,0.5,0.7,0.9},
        ytick={0.884,0.888,0.892,0.896,0.90},
        scaled y ticks = false,
        yticklabel style={/pgf/number format/fixed zerofill,/pgf/number format/precision=3},
        every axis y label/.style={at={(current axis.north west)},right=3mm,above=0mm},
        label style={font=\small},
        tick label style={font=\small},
    ]
        
    \addplot[line width=0.25mm,mark size=2pt,mark=o, color=O1] 
        plot coordinates {
(1,	0.887520907	)
(2,	0.887364833	)
(3,	0.887096993	)
(4,	0.887342977	)
(5,	0.887553433	)
(6,	0.88689951	)
(7,	0.886803304	)
(8,	0.886162207	)
(9,	0.885277105	)
};
    \addplot[line width=0.25mm,mark size=2pt,mark=square, color=B2] 
        plot coordinates {
(1,	0.892962047	)
(2,	0.893442018	)
(3,	0.893237698	)
(4,	0.893825453	)
(5,	0.893643052	)
(6,	0.893709734	)
(7,	0.895114723	)
(8,	0.89415766	)
(9,	0.894523984	)
};
    \addplot[line width=0.25mm,mark size=2pt,mark=pentagon, color=B6] 
        plot coordinates {
(1,	0.895846621	)
(2,	0.895589921	)
(3,	0.89613033	)
(4,	0.895932602	)
(5,	0.89573213	)
(6,	0.897245967	)
(7,	0.89671028	)
(8,	0.896586039	)
(9,	0.897200817	)
};
    \addplot[line width=0.25mm,mark size=2pt,mark=diamond, color=O3] 
        plot coordinates {
(1,	0.894300914	)
(2,	0.894517217	)
(3,	0.894838482	)
(4,	0.894501648	)
(5,	0.89568049	)
(6,	0.895360388	)
(7,	0.895593968	)
(8,	0.895792247	)
(9,	0.895682233	)
};

    \end{axis}
\end{tikzpicture}\hspace{0mm}%
}%
\subfloat[{AMiner}]{
\begin{tikzpicture}[scale=1]\label{subfig:AMiner3}
    \begin{axis}[
        height=\columnwidth/2.3,
        width=\columnwidth/2.0,
        ylabel={\em AUC},
        xmin=0.5, xmax=9.5,
        ymin=0.95, ymax=0.982,
        xtick={1,3,5,7,9},
        xticklabel style = {font=\footnotesize},
        xticklabels={0.1,0.3,0.5,0.7,0.9},
        ytick={0.95,0.958,0.966,0.974,0.982},
        scaled y ticks = false,
        yticklabel style={/pgf/number format/fixed zerofill,/pgf/number format/precision=3},
        every axis y label/.style={at={(current axis.north west)},right=3mm,above=0mm},
        label style={font=\small},
        tick label style={font=\small},
    ]
        
    \addplot[line width=0.25mm,mark size=2pt,mark=o, color=O1] 
        plot coordinates {
(1,	0.972659176	)
(2,	0.974115334	)
(3,	0.975412307	)
(4,	0.97583695	)
(5,	0.97548689	)
(6,	0.973213735	)
(7,	0.971520469	)
(8,	0.969285072	)
(9,	0.963755062	)
};
    \addplot[line width=0.25mm,mark size=2pt,mark=square, color=B2] 
        plot coordinates {
(1,	0.969292877	)
(2,	0.969457201	)
(3,	0.966829217	)
(4,	0.970093254	)
(5,	0.976845289	)
(6,	0.967228033	)
(7,	0.974332822	)
(8,	0.958252511	)
(9,	0.959849391	)
};
    \addplot[line width=0.25mm,mark size=2pt,mark=pentagon, color=B6] 
        plot coordinates {
(1,	0.970115906	)
(2,	0.973663297	)
(3,	0.979405883	)
(4,	0.958792433	)
(5,	0.975381359	)
(6,	0.952362422	)
(7,	0.955567082	)
(8,	0.973680485	)
(9,	0.965708988	)
};
    \addplot[line width=0.25mm,mark size=2pt,mark=diamond, color=O3] 
        plot coordinates {
(1,	0.975565942	)
(2,	0.971829275	)
(3,	0.972453751	)
(4,	0.967753723	)
(5,	0.966420491	)
(6,	0.977448479	)
(7,	0.972782141	)
(8,	0.973168112	)
(9,	0.967275731	)
};
    \end{axis}
\end{tikzpicture}\hspace{0mm}%
}%
\subfloat[{DBLP}]{
\begin{tikzpicture}[scale=1]\label{subfig:DBLP3}
    \begin{axis}[
        height=\columnwidth/2.3,
        width=\columnwidth/2.0,
        ylabel={\em AUC},
        xmin=0.5, xmax=9.5,
        ymin=0.926, ymax=0.938,
        xtick={1,3,5,7,9},
        xticklabel style = {font=\footnotesize},
        xticklabels={0.1,0.3,0.5,0.7,0.9},
        ytick={0.926,0.929,0.932,0.935,0.938},
        scaled y ticks = false,
        yticklabel style={/pgf/number format/fixed zerofill,/pgf/number format/precision=3},
        every axis y label/.style={at={(current axis.north west)},right=3mm,above=0mm},
        label style={font=\small},
        tick label style={font=\small},
    ]
        
    \addplot[line width=0.25mm,mark size=2pt,mark=o, color=O1] 
        plot coordinates {
(1,	0.93343834	)
(2,	0.933134873	)
(3,	0.932890484	)
(4,	0.932634515	)
(5,	0.932515943	)
(6,	0.932323893	)
(7,	0.931956937	)
(8,	0.931815444	)
(9,	0.929793223	)
};
    \addplot[line width=0.25mm,mark size=2pt,mark=square, color=B2] 
        plot coordinates {
(1,	0.933012704	)
(2,	0.933652542	)
(3,	0.932634416	)
(4,	0.933479018	)
(5,	0.932683894	)
(6,	0.931560064	)
(7,	0.932295246	)
(8,	0.930103466	)
(9,	0.930823955	)
};
    \addplot[line width=0.25mm,mark size=2pt,mark=pentagon, color=B6] 
        plot coordinates {
(1,	0.936893639	)
(2,	0.936015073	)
(3,	0.93637217	)
(4,	0.937137801	)
(5,	0.936037117	)
(6,	0.935853944	)
(7,	0.934942332	)
(8,	0.935881487	)
(9,	0.934840522	)
};
    \addplot[line width=0.25mm,mark size=2pt,mark=diamond, color=O3] 
        plot coordinates {
(1,	0.934757663	)
(2,	0.934586072	)
(3,	0.934234335	)
(4,	0.933923932	)
(5,	0.933252782	)
(6,	0.933661328	)
(7,	0.932133716	)
(8,	0.931219545	)
(9,	0.931552538	)
};
    \end{axis}
\end{tikzpicture}\hspace{0mm}%
}%
\subfloat[{Google}]{
\begin{tikzpicture}[scale=1]\label{subfig:Google3}
    \begin{axis}[
        height=\columnwidth/2.3,
        width=\columnwidth/2.0,
        ylabel={\em AUC},
        xmin=0.5, xmax=9.5,
        ymin=0.808, ymax=0.836,
        xtick={1,3,5,7,9},
        xticklabel style = {font=\footnotesize},
        xticklabels={0.1,0.3,0.5,0.7,0.9},
        ytick={0.808,0.815,0.822,0.829,0.836},
        scaled y ticks = false,
        yticklabel style={/pgf/number format/fixed zerofill,/pgf/number format/precision=3},
        every axis y label/.style={at={(current axis.north west)},right=3mm,above=0mm},
        label style={font=\small},
        tick label style={font=\small},
    ]
        
    \addplot[line width=0.25mm,mark size=2pt,mark=o, color=O1] 
        plot coordinates {
(1,	0.815350761	)
(2,	0.815492182	)
(3,	0.815083026	)
(4,	0.814991099	)
(5,	0.815134518	)
(6,	0.815327691	)
(7,	0.814318559	)
(8,	0.81355789	)
(9,	0.812267024	)
};
    \addplot[line width=0.25mm,mark size=2pt,mark=square, color=B2] 
        plot coordinates {
(1,	0.82642012	)
(2,	0.82598123	)
(3,	0.819993981	)
(4,	0.819722721	)
(5,	0.826309311	)
(6,	0.823331228	)
(7,	0.824849117	)
(8,	0.819864863	)
(9,	0.825323715	)
};
    \addplot[line width=0.25mm,mark size=2pt,mark=pentagon, color=B6] 
        plot coordinates {
(1,	0.830410952	)
(2,	0.828987297	)
(3,	0.82813825	)
(4,	0.832351126	)
(5,	0.833286836	)
(6,	0.82752041	)
(7,	0.832854342	)
(8,	0.831734704	)
(9,	0.826701274	)
};
    \addplot[line width=0.25mm,mark size=2pt,mark=diamond, color=O3] 
        plot coordinates {
(1,	0.830337345	)
(2,	0.830538788	)
(3,	0.830019728	)
(4,	0.825152653	)
(5,	0.827569936	)
(6,	0.825146223	)
(7,	0.828140709	)
(8,	0.827594704	)
(9,	0.827797319	)
};
    \end{axis}
\end{tikzpicture}\hspace{0mm}%
}%
\subfloat[{MAG}]{
\begin{tikzpicture}[scale=1]\label{subfig:MAG3}
    \begin{axis}[
        height=\columnwidth/2.3,
        width=\columnwidth/2.0,
        ylabel={\em AUC},
        xmin=0.5, xmax=9.5,
        ymin=0.982, ymax=0.99,
        xtick={1,3,5,7,9},
        xticklabel style = {font=\footnotesize},
        xticklabels={0.1,0.3,0.5,0.7,0.9},
        ytick={0.982,0.984,0.986,0.988,0.99},
        scaled y ticks = false,
        yticklabel style={/pgf/number format/fixed zerofill,/pgf/number format/precision=3},
        every axis y label/.style={at={(current axis.north west)},right=3mm,above=0mm},
        label style={font=\small},
        tick label style={font=\small},
    ]
        
    \addplot[line width=0.25mm,mark size=2pt,mark=o, color=O1] 
        plot coordinates {
(1,	0.987742151	)
(2,	0.987673935	)
(3,	0.987534033	)
(4,	0.987143237	)
(5,	0.986937501	)
(6,	0.98652488	)
(7,	0.985877993	)
(8,	0.984979937	)
(9,	0.983323832	)
};
    \addplot[line width=0.25mm,mark size=2pt,mark=square, color=B2] 
        plot coordinates {
(1,	0.987911595	)
(2,	0.987509483	)
(3,	0.987333919	)
(4,	0.986891246	)
(5,	0.987169168	)
(6,	0.986504818	)
(7,	0.986035645	)
(8,	0.985470698	)
(9,	0.983749572	)
};
    \addplot[line width=0.25mm,mark size=2pt,mark=pentagon, color=B6] 
        plot coordinates {
(1,	0.989313977	)
(2,	0.989313425	)
(3,	0.989012713	)
(4,	0.988917449	)
(5,	0.989061025	)
(6,	0.988555922	)
(7,	0.987900064	)
(8,	0.987727404	)
(9,	0.98631957	)
};
    \addplot[line width=0.25mm,mark size=2pt,mark=diamond, color=O3] 
        plot coordinates {
(1,	0.987694665	)
(2,	0.987470885	)
(3,	0.987578804	)
(4,	0.987274423	)
(5,	0.986716645	)
(6,	0.986520999	)
(7,	0.985500752	)
(8,	0.985092359	)
(9,	0.983939281	)
};
    \end{axis}
\end{tikzpicture}\hspace{0mm}%
}
\end{small}
\vspace{-3mm}
\caption{Varying $\alpha$.} \label{fig:param3}
\vspace{-2mm}
\end{figure*}

\subsection{Edge Classification}

Table \ref{tbl:classification-perf} presents the edge classification performance of all methods on five datasets. Overall, our proposed methods consistently outperform all competitors on all five datasets. On review datasets like Amazon and Google, our method (\algo (\dvffp)) using the max aggregator achieves approximately 0.9\%-3.1\% improvement in AP and approximately 0.4\%-0.7\% improvement in AUC compared to the best competitors. On citation network datasets like AMiner using keywords as edge labels, our method (\algo (\ffp)) achieves around 177.5\% improvement in AP and around 1.5\% improvement in AUC compared to the best competitors. Note that most of the baselines cannot achieve high AP (below 0.2), due to the difficulties of classifying keywords in AMiner, as the number of keyword labels in the original dataset is much higher than the number of labels in the other datasets. On citation network datasets like DBLP and MAG using the field of study as edge labels, our method (\algo (\dvffp)) using the max aggregator achieves approximately 7.1\%-15.0\% improvement in AP and approximately 1.2\%-1.9\% improvement in AUC compared to the best competitors. 
It is worth noting that for datasets like Amazon and AMiner, FC performs the best among the competitors. This indicates the difficulties of capturing the structural similarities in these graph datasets. However, our methods can still successfully generate better edge representations on these datasets. 
Another observation is that \algo (\dvffp) outperforms \algo (\ffp) and other competitors on four datasets out of five. This suggests the importance of introducing intermediate edge representation independently from the views of heterogeneous node sets $\ZM_\U$ and $\ZM_\V$.



 \subsection{Parameter Analysis}

In this section, we experimentally study the effects of varying three key parameters in our proposed method, including $\beta$ used in Eq. \eqref{eq:P}, $k$ for matrix $\QM$ dimension, and $\alpha$ used in Eq. \eqref{eq:obj}. We report the AUC scores by \algo (\ffp) and \algo (\dvffp) with three different aggregators 
when varying these parameters. 

In Figure \ref{fig:param}, we report how AUC scores vary for different $\beta$ for \algo (\ffp) and different $\gamma$ for \algo (\dvffp). For Amazon, Google, and MAG, \algo (\dvffp) with the max aggregator consistently performs the best across different $\beta$. For AMiner, the AUC scores reach the maximum when $\beta=0.7$ and then decrease for \algo (\dvffp) with max aggregator. For DBLP, the \algo (\ffp) performs the best when $\beta=0.2$, but when $\beta\ge 0.3$, \algo (\dvffp) with max aggregator becomes the best among all our methods. We can also observe that on DBLP, Google, and MAG, \algo (\ffp) is more sensitive to the change of $\beta$ compared with \algo (\dvffp).

In Figure \ref{fig:param2}, we report how AUC scores vary for different SVD dimensions $k$. $\infty$ in Figure \ref{fig:param2} refers to using power iteration to solve for the edge embeddings. For Amazon, DBLP, and Google, the AUC scores of all our methods increase as $k$ increases (excluding the solution from power iteration), as larger embedding dimensions can contain more graph structural information. For AMiner, AUC reaches the maximum when $k=8$ for \algo (\dvffp) using the sum aggregator. For MAG, AUC reaches the maximum for \algo (\dvffp) using the max aggregator when $k=128$ and then decreases.

In Figure \ref{fig:param3}, we report how AUC scores vary for different $\alpha$. By tuning $\alpha$, we observe that \algo (\dvffp) with the max aggregator can achieve the best performance on these datasets compared with other methods. In particular, on Amazon, AMiner, DBLP, and Google, \algo (\dvffp) with the max aggregator performs best with $\alpha$ value between 0.3 and 0.6. On MAG, \algo (\dvffp) with the max aggregator shows a decreasing trend as $\alpha$ increases. As mentioned in Section \ref{sec:obj}, $\alpha$ balances the importance between the edge representations derived from edge features and graph structures. This suggests on MAG, our methods can achieve improvements by a careful trade-off between edge attributes and graph structures.

\section{Related Work}\label{sec:relatedwork}

This section reviews existing graph node/edge representation learning on unipartite/bipartite graphs, as well as their applications.

\subsection{Node-wise Representation Learning}
Node-wise GRL refers to the process of generating embeddings for the nodes of a graph. Conventional approaches for addressing this task involve methods based on matrix factorization and random walk. In matrix factorization-based methods, such as HOPE~\cite{Ou}, AROPE~\cite{Zhang2018}, PRONE~\cite{Zhang2019}, NRP~\cite{Yang2019}, PANE~\cite{Yang2021}, and SketchNE~\cite{Xie2023}, a proximity-based matrix $\PM$ is initially created for the graph, where each element 
$\PM[i,j]$ denotes the proximity measure between nodes $i$ and $j$. Subsequently, a dimension reduction technique is employed to derive lower-dimensional representations for the nodes. In random walk-based methods, such as Deepwalk~\cite{Perozzia}, LINE~\cite{Tang}, and node2vec~\cite{Grover}, the process begins with the generation of random walks for each node to capture the underlying graph structures. Subsequently, the co-occurrence in these random walks is employed to assess node similarities and generate node embedding vectors.

Another line of research lies in graph neural networks (GNNs). The major categories of GNNs, for example, GCN~\cite{Kipf2016}, GraphSAGE~\cite{Hamilton2017}, SGC~\cite{Wu2019}, DGI~\cite{velivckovic2018deep}, GAT~\cite{Velickovivelickovic}, and GATv2~\cite{brody2022how}, adopt ideas from convolutional neural networks for modeling graph-structured data. GNNs aggregate local neighborhood information to get contextual representation for graph nodes and have shown promising results in this area. To consider the effect of edge attributes, some new GNN models are proposed to incorporate them during the training process.  
EGAT~\cite{wang2021egat} proposes edge-featured graph attention layers that can accept node and edge features as inputs and handle them spontaneously within the models.
GERI~\cite{sun2019novel} constructs a heterogeneous graph using the attribute information and applies random walk with a modified heterogeneous skip-gram to learn node embeddings. 
EEGNN~\cite{liu2023eegnn} proposes a framework called edge-enhanced graph neural network that uses the structural information extracted from a Bayesian nonparametric model for graphs to consider the effect of self-loops and multiple edges between two nodes and improve the performance of various deep message-passing GNNs.
EGNN~\cite{gong2019exploiting} uses multi-dimensional nonnegative-valued edge features represented as a tensor and applies GCN/GAT to exploit multiple attributes associated with each edge.
GraphBEAN~\cite{fathony2023interaction} applies autoencoders on bipartite graphs with both node and edge attributes to obtain node embeddings for node and edge level anomaly detection.


\subsection{Edge-wise Representation Learning}
Edge-wise GRL refers to the process of generating embeddings for
edges of a graph. 
\cite{li2018edge} uses random walks to sample a series of edge sequences to generate edge embeddings and apply clustering algorithms for community detection. 
Edge2Vec~\cite{gao2019edge2vec} uses deep auto-encoders and
skip-gram models to generate edge embeddings that preserve both
the local and global structure information of edges for biomedical knowledge discovery. 
AttrE2Vec~\cite{bielak2022attre2vec} generates random walks starting from a node and uses aggregation functions to aggregate node/edge features in the random walks and obtain node/edge representations. Then, it uses auto-encoders and self-attention networks with feature reconstruction loss and graph structural loss to build edge embeddings in an unsupervised manner.
\cite{wang2023efficient} uses matrix factorization and feature aggregation to generate edge representation vectors based on the graph structure surrounding edges and edge attributes, which can encode high-order proximities of edges and edge attribute information into low-dimensional vectors.
CEN-DGCNN~\cite{zhou2023co} introduces a deep graph convolutional neural network that integrates node and edge features, preventing over-smoothing. It captures non-local structural features and refines high-order node features by considering long-distance dependencies and multi-dimensional edge features.
DoubleFA~\cite{10.1007/978-3-031-39831-5_6} proposes to use top-k Personalized PageRank to conduct proximal feature aggregation and anomaly feature aggregation using edge features for edge anomaly detection.

\subsection{Bipartite Graph Representation Learning}
For a comprehensive review of existing bipartite graph representation learning methods, we suggest readers check \cite{Giamphy2023}.
BiANE~\cite{Huang2020a} employs auto-encoders to model inter-partition and intra-partition proximity using attribute proximity and structure proximity through a latent correlation training approach. 
Cascade-BGNN~\cite{Zhu2020b} utilizes customized inter-domain message passing and intra-domain alignment with adversarial learning for message aggregation across and within graph partitions. 
BiGI~\cite{10.1145/3437963.3441783} utilizes GCN to generate initial node embeddings and applies a subgraph-level attention mechanism to maximize the mutual information between local and global node representations.
DualHGCN~\cite{10.1145/3442381.3449954} transforms the multiplex bipartite network into two sets of homogeneous hypergraphs and uses spectral hypergraph convolutional operators to capture information within and across domains.
GEBE~\cite{yang2022scalable} proposes proximity matrices derived from the edge weight matrix and applies matrix factorization to capture multi-hop similarity/proximity between homogeneous/heterogeneous nodes. 
AnchorGNN~\cite{10.14778/3626292.3626300} proposes a global-local learning framework that leverages an anchor-based message-passing schema to generate node embeddings for large bipartite graphs.

\section{Conclusions}\label{sec:conclude}
In this work, we introduce the problem of ERL on EABGs and propose \algo models to address this problem. Building on an in-depth theoretical analysis of extending the feature propagation paradigm in GNNs to ERL on EABGs, we design the FFP scheme that is able to effectively capture long-range dependencies between edges for generating high-quality edge representations without entailing vast computational costs. On the basis of FFP, we propose the dual-view FFP by leveraging the semantics of two sets of heterogeneous nodes in the input bipartite graphs to enhance edge representations. The effectiveness of our proposed \algo models is validated by our extensive experiments comparing \algo against nine baselines over five real datasets.

\begin{acks}
This research is supported by the Ministry of Education, Singapore, under its MOE AcRF TIER 3 Grant (MOE-MOET32022-0001). Any opinions, findings and conclusions or recommendations expressed in this material are those of the author(s) and do not reflect the views of the Ministry of Education, Singapore. Renchi Yang is supported by the NSFC Young Scientists Fund (No. 62302414) and the Hong Kong RGC ECS grant (No. 22202623).
\end{acks}

\appendix
\section{Proofs}\label{sec:proof}

\begin{proof}[\bf Proof of Lemma \ref{lem:PUPV}]
We first prove that $\PM_\U$ is a row-stochastic matrix. Since it is symmetric, then its doubly stochastic property naturally follows.
By Eq. \eqref{eq:PU-PV}, the $(i,j)$-th entry of $\PM_\U$ is
\begin{equation*}
\sum_{e_j\in \EDG}\PM_\U[e_i,e_j]=\sum_{e_j\in \EDG} \frac{1}{\DM[u^{(i)},u^{(i)}]\cdot \mathbb{1}_{u^{(i)}\in e_j} }=1, 
\end{equation*}
where $\mathbb{1}_{u^{(i)}\in e_j}$ is an indicator function which equals $1$ when node $u^{(i)}$ is an endpoint of edge $e_j$. A similar proof can be done for $\PM_\V$.
\end{proof}

\begin{lemma}\label{lem:sigma}
If $\sigma_1$ be the largest singular value of $\EM\DM^{-\frac{1}{2}}$, $\sigma_1\le 1$.
\begin{proof}
According to \cite{strang2022introduction}, The singular values are the square roots of the non-zero eigenvalues of $\EM\DM^{-\frac{1}{2}}\cdot (\EM\DM^{-\frac{1}{2}})^{\top}=\PM$. This implies that all the eigenvalues of $\PM$ are non-negative. Thus, $\sigma_1 = \sqrt{\lambda_1}$, where $\lambda_1$ is the largest eigenvealue of $\PM$. Recall that $\PM$ is doubly stochastic. Then, we have 
$$\|\PM \|_\infty= \underset{{1\le i\le |\EDG|}}{\max}{\sum_{j=1}^{|\EDG|}\PM_{i,j}}=1.$$
By Theorem 5.6.9 in \cite{horn2012matrix},
\begin{equation*}
|\lambda_1| \le \rho(\PM) \le \|\PM \|_\infty =1,
\end{equation*}
where $\rho(\PM)$ is the spectral radius of $\PM$, which leads to $\sigma_1 \le 1$.
\end{proof}
\end{lemma}

\begin{proof}[\bf Lemma \ref{lem:closed-form-sol}]
First, it is easy to derive that Eq.~\ref{eq:reg} can be transformed to its equivalent form $\OO_r = trace(\ZM^{\top}(\IM-\PM)\ZM)$. Accordingly, Eq.~\ref{eq:obj} is converted into
\begin{equation*}
(1-\alpha)\cdot\|\ZM - f_{\Theta}(\XM)\|^2_F + \alpha\cdot trace(\ZM^{\top}(\IM-\PM)\ZM).
\end{equation*}
By setting its derivative w.r.t. $\ZM$ to zero, we obtain the optimal $\ZM$ as:
\begin{align}
& \frac{\partial{\{(1-\alpha)\cdot\|\ZM - \XM\|^2_F + \alpha \cdot trace(\ZM^{\top}(\IM-\PM)\ZM)\}}}{\partial{\ZM}}=0 \notag\\
& \Longrightarrow (1-\alpha)\cdot(\ZM - f_{\theta}(\XM)) + \alpha (\IM-\PM)\ZM = 0 \notag\\
& \Longrightarrow \ZM = (1-\alpha)\cdot \left(\IM-\alpha\PM\right)^{-1}\cdot f_{\Theta}(\XM). \label{eq:Z-derivative}
\end{align}
By the definition of Neumann series, \ie $(\IM-\MM)^{-1}=\sum_{t=0}^{\infty}{\MM^t}$, we have 
$$(\IM - \alpha \PM)^{-1} = \sum_{t=0}^{\infty}{\alpha^t\cdot \PM^{t}}.$$
Plugging the above equation into Eq. \eqref{eq:Z-derivative} completes the proof.
\end{proof}

\begin{proof}[\bf Proof of Lemma \ref{lem:PPP}]
Recall that $\DM$ can be represented by 
\[
\begin{pmatrix}
\DM_\U
  & \rvline & \bigzero \\
\hline
  \bigzero & \rvline &
\DM_\V
\end{pmatrix}
\]
Then, as per Eq. \eqref{eq:P}, we have
\begin{align*}
\PM &= (\sqrt{\beta}\EM_{\U} \mathbin\Vert  \sqrt{1-\beta} \EM_{\V})\cdot \DM^{-1}\cdot (\sqrt{\beta} \EM_{\U} \mathbin\Vert  \sqrt{1-\beta} \EM_{\V})^{\top}\\
&= \left(\sqrt{\beta} \EM_{\U}\DM_\U^{-\frac{1}{2}} \mathbin\Vert  \sqrt{1-\beta} \EM_{\V}\DM_\V^{-\frac{1}{2}}\right)\cdot \left(\sqrt{\beta} \EM_{\U}\DM_\U^{-\frac{1}{2}} \mathbin\Vert  \sqrt{1-\beta} \EM_{\V}\DM_\V^{-\frac{1}{2}}\right)^{\top}\\
&= {\beta}\cdot \EM_{\U}\DM_\U^{-1}\EM_\U^{\top} + (1-\beta) \cdot \EM_{\V}\DM_\V^{-1}\EM_\V^{\top}\\
&= \beta\cdot\PM_\U + (1-\beta)\cdot \PM_\V,
\end{align*}
which finishes the proof.
\end{proof}

\begin{proof}[\bf Proof of Eq. \eqref{eq:Z-2-parts}]
According to Eq. \eqref{eq:Z}, we have
\begin{equation*}
\left((1-\alpha)\sum_{t=0}^{T_{mix}-1}{\alpha^t \PM^{t}} \cdot f_{\Theta}(\XM) \right) + \left( (1-\alpha)\sum_{t=T_{mix}}^{\infty}{\alpha^t \boldsymbol{\Pi} }\right).
\end{equation*}
Note that we can rewrite the second part as:
\begin{align*}
 (1-\alpha)\sum_{t=T_{mix}}^{\infty}{\alpha^t \boldsymbol{\Pi} } & =  \left((1-\alpha)\sum_{t=0}^{\infty}{\alpha^t} - (1-\alpha)\sum_{t=0}^{T_{mix}-1}{\alpha^t}\right) \cdot \boldsymbol{\Pi} \\
 & = (1- (1-\alpha^{T_{mix}}))\cdot \boldsymbol{\Pi} = \alpha^{T_{mix}}\boldsymbol{\Pi}.
\end{align*}
Eq. \eqref{eq:Z-2-parts} naturally follows.
\end{proof}

\begin{proof}[\bf Proof of Lemma \ref{lem:mix}]
We first prove that $\sigma_2^2$ is the second largest eigenvalue of $\PM$. By Eq. \eqref{eq:P}, $\PM=(\EM\DM^{-1/2})\cdot (\EM\DM^{-1/2})^{\top}$, which indicates that the eigenvalues of $\PM$ are the squared singular values of $\EM\DM^{-1/2}$ \cite{horn2012matrix,strang2022introduction}. Since singular values are non-negative, the second largest eigenvalue of $\PM$ is $\sigma_2^2$.
According to the fact of $\PM$ is a reversible Markov chain and Theorem 12.5 in \cite{levin2017markov}, $T_{mix}$ satisfies $T_{mix} \ge \frac{1}{1-\sigma_2^2}-1$.
\end{proof}

\begin{proof}[\bf Proof of Theorem \ref{lem:QQS}]
We first consider that $\UM\SVM\VM^{\top}$ is the exact full SVD of $\EM\DM^{-1/2}$. According to Lemma \ref{lem:sigma}, we can get
\begin{align*}
\QM\QM^{\top} &= \UM\cdot {\frac{1}{1-\alpha\SVM^2}}\cdot \UM^{\top} = \sum_{t=0}^{\infty} {\alpha^t \cdot \UM\cdot (\SVM^2)^t\cdot \UM^{\top} }.
\end{align*}
Since $\UM$ and $\VM$ are semi-unitary matrices, i.e., $\UM^{\top}\UM$ and $\VM^{\top}\VM$,
\begin{align*}
\QM\QM^{\top} & = (1-\alpha)\sum_{t=0}^{\infty} {\alpha^t  (\UM\SVM^2\UM^{\top})^t } = \sum_{t=0}^{\infty} {\alpha^t  (\UM\SVM\cdot (\VM^{\top}\VM)\cdot \SVM\UM^{\top})^t } \\
& = (1-\alpha)\sum_{t=0}^{\infty} {\alpha^t  (\EM\DM^{-\frac{1}{2}} \cdot (\EM\DM^{-\frac{1}{2}})^{\top})^t }\\
& = (1-\alpha)\sum_{t=0}^{\infty} {\alpha^t  (\EM\DM^{-1}\EM^{\top})^t } = (1-\alpha)\sum_{t=0}^{\infty} {\alpha^t  \PM^t }.
\end{align*}

According to \cite{horn2012matrix,strang2022introduction}, the definition of $\PM$ in Eq. \eqref{eq:P} (i.e., $\PM=\EM\DM^{-1/2}\cdot (\EM\DM^{-1/2})$) implies that the singular values of $\EM\DM^{-{1}/{2}}$ are the square roots of the eigenvalues of $\PM$, and the left singular vectors of $\EM\DM^{-\frac{1}{2}}$ are the eigenvectors of $\PM$. In particular, due to the non-negativity of singular values, the $k$-th largest eigenvalue of $\PM$ is equal to $\sigma_k^2$ where $\sigma_k$ denotes the $k$-th largest singular value of $\EM\DM^{-{1}/{2}}$.

Recall that $\PM$ is doubly stochastic, meaning that $\PM$ is a symmetric matrix. Using Theorem 4.1 in \cite{Zhang2018}, we can derive that the singular values of $\PM$ are the absolute values of the corresponding eigenvalues, and the left singular vectors of $\PM$ are equal to the eigenvectors of $\PM$. 
Since all the eigenvalues of $\PM$ are non-negative, its $k$-th largest eigenvalue is equal to the $k$-th largest singular value of $\PM$.

Combining the above two conclusions, we can extrapolate that the $k$-th largest singular value of $\PM$ is equal to $\sigma_k^2$, and the left singular vectors of $\EM\DM^{-{1}/{2}}$ are the left singular vectors of $\PM$.

\begin{theorem}[Eckart–Young Theorem \cite{gloub1996matrix}]\label{lem:eym}
Suppose that $\MM_{k}\in\mathbb{R}^{n\times k}$ is the rank-$k$ approximation to $\MM\in\mathbb{R}^{n\times n}$ obtained by exact SVD, then
$$\min_{rank(\widehat{\MM})\le k}{\|\MM-\widehat{\MM}\|_F}=\|\MM-\MM_{k}\|_F=\sigma_{k+1},$$
where $\sigma_{i}$ represents the $i$-th largest singular value of $\MM$.
\end{theorem}

Let $\UM\SVM\VM^{\top}$ be the exact top-$k$ SVD of $\EM\DM^{-{1}/{2}}$. Then, $\UM\SVM^2\UM^{\top}$ is the exact top-$k$ SVD of $\PM$. By leveraging Eckart–Young Theorem in Theorem \ref{lem:eym}, we obtain
\begin{equation*}
\|\UM\SVM^2\UM^{\top} - \PM\|_F \le \sigma_k^2.
\end{equation*}
Next, we prove that $\UM$ and $\frac{1}{1-\alpha\SVM^2}$ are the top-$k$ left singular vectors and singular values of $(1-\alpha)\sum_{t=0}^{\infty} {\alpha^t  \PM^t }$, respectively. We assume $k=|\EDG|$, which means 
$$\UM {\frac{1}{1-\alpha\SVM^2}} \UM^{\top} = (1-\alpha)\sum_{t=0}^{\infty} {\alpha^t  \PM^t }.$$
Consider vector $\UM[:,i]$ and scalar $\frac{1}{1-\alpha\SVM[i,i]^2}$ and denote by $\mathbf{e}_i$ a one-hot vector. We can derive 
\begin{align*}
(1-\alpha)\sum_{t=0}^{\infty} {\alpha^t  \PM^t }\cdot \UM[:,i] &= \UM {\frac{1}{1-\alpha\SVM^2}} \UM^{\top} \cdot \UM[:,i]\\
&=\UM {\frac{1}{1-\alpha\SVM^2}}\cdot \mathbf{e}_i\\
&=\frac{1}{1-\alpha\SVM[i,i]^2}\cdot \UM[:,i].
\end{align*}
In addition, by Lemma \ref{lem:sigma}, $\frac{1}{1-\alpha\SVM[i,i]^2}$ is non-negative and is monotonically decreasing with $i$. As a consequence, we can conclude that $\UM[:,i]$ and scalar $\frac{1}{1-\alpha\SVM[i,i]^2}$ are an eigenpair of $(1-\alpha)\sum_{t=0}^{\infty} {\alpha^t  \PM^t }$, and thus, its $i$-th largest left singular vector and singular value. Then, using Theorem \ref{lem:eym} yields
\begin{align*}
\|\QM\QM^{\top} - (1-\alpha)\sum_{t=0}^{\infty} {\alpha^t  \PM^t } \|_F & = \|\UM\frac{1}{1-\alpha\SVM^2}\UM^{\top} - (1-\alpha)\sum_{t=0}^{\infty} {\alpha^t  \PM^t } \|_F\\
& \le \frac{1}{1-\alpha \sigma_k^2},
\end{align*}
which finishes the proof.
\end{proof}

\balance
\bibliographystyle{ACM-Reference-Format}
\bibliography{sample}

\end{document}